%% file: main.tex
\documentclass[twoside]{article}

\usepackage{alias}
\usepackage[accepted]{aistats2023}
\usepackage[round]{natbib}

\begin{document}

%

%

\twocolumn[

\aistatstitle{Mixed-Effect Thompson Sampling}

\aistatsauthor{Imad Aouali \And Branislav Kveton \And Sumeet Katariya}

\aistatsaddress{Criteo AI Lab \And Amazon \And Amazon}
]


\begin{abstract}
A contextual bandit is a popular framework for online learning to act under uncertainty. In practice, the number of actions is huge and their expected rewards are correlated. In this work, we introduce a general framework for capturing such correlations through a mixed-effect model where actions are related through multiple shared effect parameters. To explore efficiently using this structure, we propose \textbf{M}ixed-\textbf{E}ffect \textbf{T}hompson \textbf{S}ampling (\alg) and bound its Bayes regret. The regret bound has two terms, one for learning the action parameters and the other for learning the shared effect parameters. The terms reflect the structure of our model and the quality of priors. Our theoretical findings are validated empirically using both synthetic and real-world problems. We also propose numerous extensions of practical interest. While they do not come with guarantees, they perform well empirically and show the generality of the proposed framework.
\end{abstract}

\input{introduction}

\input{setting}

\input{algorithm}

\input{analysis}

\input{experiments}

\input{related_work}

\input{conclusions}


\bibliographystyle{plainnat}
\bibliography{Brano,references}


\appendix
\onecolumn

\input{appendix}

\vfill

\end{document}

%% file: introduction.tex
\section{INTRODUCTION}
\label{sec:introduction}

A \emph{contextual bandit} \citep{mab_book,lattimore19bandit,li10contextual,chu11contextual} is a popular sequential decision-making framework where an \emph{agent} interacts with an environment over $n$ rounds. In each round, the agent observes a \emph{context}, takes an \emph{action}, and receives a \emph{reward} that depends on both the context and the taken action. The goal of the agent is to maximize the expected cumulative reward over $n$ rounds. Since the expected rewards of actions are unknown, the agent must balance between taking the action that maximizes the estimated reward using collected data (\emph{exploitation}), and exploring other actions to improve their estimates (\emph{exploration}). This trade-off is often addressed using either \emph{upper confidence bounds (UCBs)} \citep{auer02finitetime} or \emph{Thompson sampling (TS)} \citep{thompson33likelihood}. As an example, in online advertising, contexts can be features of users, actions can be products, and the expected reward can be the click-through rate (CTR).

Efficient exploration in contextual bandits \citep{langford08epochgreedy,dani08stochastic,li10contextual,abbasi-yadkori11improved,agrawal13thompson} is an important research direction, as their action space is usually huge and naive exploration may lead to suboptimal performance. In this work, we start from the basic observation that the expected rewards of actions in real-world problems are often correlated. To model this phenomenon, we study a structured mixed-effect bandit environment where each \emph{action parameter} depends on one or multiple \emph{effect parameters}. The actions are related through the effect parameters. Therefore, taking an action teaches the agent about its effect parameters, which consequently teaches it about other actions that share the same effect parameters. We present three motivating examples next.

\textbf{Movie recommendation:} Here we want to recommend a movie to a user with the highest expected rating. User $j$ and movie $i$ are represented by vectors $x_j$ (context) and $\theta_i$ (action parameter), respectively. The expected rating that user $j$ gives to movie $i$ is $x_j^\top \theta_i$. We assume that the vector $x_j$ is observed. Then the most natural idea is to learn all $\theta_i$ individually using classic bandit methods \citep{li10contextual,chu11contextual}. This is statistically inefficient since the number of movies is often high. Fortunately, the movies could be organized into $L$ categories and such information can be leveraged to explore efficiently. We present three approaches \textbf{(A)}, \textbf{(B)} and \textbf{(C)} that do this next.

\begin{enumerate*}[label=\textbf{(\Alph*)}]
  \item For each category $\ell \in [L]$, a parameter $\psi_\ell$ is learned online using all interactions with the movies in category $\ell$. The parameter $\psi_\ell$ represents all the movies in category $\ell$ and is used instead of their individual $\theta_i$. Therefore, this approach has a high bias, as all movies in the same category are assumed to have the same expected rating. This issue can be addressed by a better model.
  \item We model each movie parameter $\theta_i$ as a random variable centered in its category parameter $\psi_\ell$. Now movies in the same category no longer have the same expected rating due to the additional uncertainty. Both the category parameters $\psi_\ell$ and movie parameters $\theta_i$ are learned online. The former is learned using all interactions with the movies in category $\ell$, while the latter is learned using all interactions with movie $i$ conditioned on $\psi_\ell$. The category parameter $\psi_\ell$ is learned using more data, which helps to learn $\theta_i$ more efficiently. This is a special case of our setting and can also be viewed as extending hierarchical Bayesian bandits \citep{hong22hierarchical} to multiple hierarchies. 
  \item The shortcoming of \textbf{(B)} is that each movie belongs to a \emph{single category}, which is unrealistic. To address this issue, we allow movies to belong to \emph{multiple categories} and then proceed as in \textbf{(B)}. To make the connection with our terminology, the categories $\ell \in [L]$ are the effects, their parameters $\psi_\ell$ are the effect parameters, and the movie parameters $\theta_i$ are the action parameters.
\end{enumerate*}

\textbf{Ad placement:} Here the agent selects a list (or \emph{slate}) of $M$ items from a catalog of $L$ items with the objective of maximizing the CTR. We assume that the agent only receives binary bandit feedback that indicates whether the user clicked on \emph{one of the items in the slate} \citep{ijcai2019-308,RejwanM20}. Again, user $j$ and slate $i$ are represented by $x_j$ (context) and $\theta_i$ (action parameter), respectively. The corresponding CTR is $f(x_j^\top \theta_i)$, where $f$ is the sigmoid function. The set of slates (of size $K \approx L^M$) is exponentially large, which makes learning $\theta_i$ individually difficult. Fortunately, the slates are related through a much smaller set of items (of size $L$). Therefore, the slates with common items can teach the agent about each other, which can be used to explore efficiently.

Efficient exploration is achieved by decomposing the parameter of slate $i$ as $\theta_i = \sum_{\ell \in [L]} b_{i, \ell} \psi_\ell + \epsilon_i$. Here $\psi_\ell \in \real^d$ is the parameter of item $\ell$ and $b_{i, \ell} \in \real$ is a mixing weight that captures position biases. That is, $b_{i, \ell} = 0$ if item $\ell$ is not in slate $i$, and $b_{i, \ell}$ is high if item $\ell$ is ranked high in slate $i$. This captures the fact that the probability of a click on an item is biased by its position in the slate, and such bias can be estimated offline. Finally, $\epsilon_i$ is a random noise that can incorporate uncertainty due to \emph{model misspecification}, for instance due to an estimation error of $b_{i, \ell}$. The benefit of this decomposition is that the parameter of item $\ell$, $\psi_\ell$, is learned using all interactions with the slates with item $\ell$. The slate parameter $\theta_i$ is learned using all interactions with slate $i$ conditioned on $\psi_\ell$. This is more statistically efficient than learning $\theta_i$ individually, which only uses the interactions with slate $i$.

\textbf{Drug design:} Here the goal is to find the optimal drug design in clinical trials \citep{pmlr-v85-durand18a}. Subject $j$ and drug $i$ are represented by vectors $x_j$ and $\theta_i$, respectively, and the expected efficacy of drug $i$ for subject $j$ is $x_j ^\top \theta_i$. Again, the most natural idea is to learn all drug parameters $\theta_i$ individually. This leads to a statistical inefficiency though when the number of candidate drugs is high. Fortunately, we can leverage the fact that drug candidates in the same trial often share components to explore efficiently. Precisely, a drug is a combination of multiple components, each with a specific dosage. Each component $\ell$ is represented by a parameter $\psi_\ell$, and the drug parameter $\theta_i$ is a \emph{known} combination of the component parameters $\psi_\ell$ weighted by their dosage. That is, $\theta_i = \sum_{\ell \in [L]} b_{i,\ell} \psi_\ell + \epsilon_i$, where $b_{i, \ell}$ is the dosage of component $\ell$ in drug $i$ and $\epsilon_i$ is a random noise to incorporate uncertainty due to model misspecification. The efficacy of each component has an \emph{effect} on the overall efficacy of the drug and is boosted by the dosage.

In all examples, we assume an underlying structure among the actions, that they are affected by multiple effects. In some problems, it is known how the effect arises. For instance, in the drug design, the actions are the drugs and the effects are their components. The mixing weight that relates an action (drug) to an effect (component) is the dosage of that component in the drug. In other problems, it may not be apparent how the effect arises and this has to be learned. We discuss this in detail in \cref{subsec:creating_structure}.

We make the following contributions. \begin{enumerate*}[label=\textbf{\arabic*)}]
  \item We formalize a general mixed-effect bandit framework represented by a two-level graphical model where each action is associated with a $d$-dimensional parameter that depends on \emph{one or multiple} effect parameters.
  \item We design mixed-effect Thompson sampling (\alg), which leverages this structure to be both statistically and computationally efficient. Despite the complex structure, we show that closed-form posteriors can be derived for Gaussian instances and efficient approximations exist in more general cases.
  \item We prove that the Bayes regret of \alg is bounded by a sum of two terms: one is associated with learning the action parameters and the other quantifies the cost of learning the effect parameters. Both terms reflect the structure of the environment and the quality of priors.
  \item We show empirically that \alg and its variants perform extremely well, and are computationally efficient in both synthetic and real-world problems.
  \item Several extensions of practical interest are given in \cref{app:extensions}. While they are not analyzed, they enjoy very good empirical performance.
\end{enumerate*}
 
Our setting is more general than previously studied hierarchical models (\cref{sec:related work}) where action parameters are centered at a single latent variable. Thus \alg has a wider range of applications, for which we gave three examples. Our algorithm is general and we provide posterior derivations that are valid for any distribution class. To showcase the generality, we also experiment with \alg on bandit problems with non-linear rewards. This also goes beyond prior works (\cref{sec:related work}) that often considered closed-form Gaussian posteriors only.

%% file: setting.tex
\section{SETTING}
\label{sec:setting}

For any positive integer $n$, we define $[n] = \{1, \dots ,n\}$. The $i$-th coordinate of vector $v$ is $v_i$. If the vector is already indexed, such as $v_j$, we write $v_{j, i}$. Let $a_1, \ldots, a_n \in \real^d$ be $n$ vectors. We denote by $a = (a_i)_{i \in [n]} \in \real^{nd}$ a vector of length $n d$ obtained by concatenating $a_1, \ldots, a_n$. We use $\otimes$ to denote the Kronecker product. The derivative of a univariate function $f$ is denoted by $\dot{f}$.

\subsection{Mixed-Effect Bandit}
\label{sec:model}

We study a setting where an agent interacts with a \emph{contextual bandit} over $n$ rounds. In round $t \in [n]$, the agent observes \emph{context} $X_t \in \cX$, where $\cX \subseteq \mathbb{R}^d$ is a $d$-dimensional \emph{context space}. After that, it takes an \emph{action} $A_t$ from an \emph{action set} $[K]$, and then observes a \emph{stochastic reward} $Y_t \in \mathbb{R}$ that depends on both $X_t$ and $A_t$. We consider a structured problem where the expected rewards of actions are correlated. Specifically, each action $i \in [K]$ is associated with an \emph{unknown $d$-dimensional action parameter} $\theta_{*, i} \in \real^d$. The correlations between the action parameters arise because they are derived from $L$ shared \emph{unknown $d$-dimensional effect parameters}, $\psi_{*, \ell} \in \real^d$ for $\ell \in [L]$. Specifically, we assume that the action parameter $\theta_{*, i}$ is sampled from the \emph{action prior distribution} $P_{0, i}$ as $ \theta_{*, i} \mid  \Psi_{*} \sim P_{0, i}(\cdot \mid \Psi_*)$, where $\Psi_* = (\psi_{*, \ell})_{\ell \in [L]} \in \real^{Ld}$ is a concatenation of the effect parameters. The distribution $P_{0, i}$ can capture sparsity, when $\theta_{*, i}$ depends only on a subset of $\Psi_*$; and also incorporate uncertainty due to model misspecification, when $\theta_{*, i}$ is not a deterministic function of $\Psi_{*}$. Finally, the effect parameters $\Psi_{*}$ are sampled from a \emph{joint effect prior} $Q_{0}$, which is known by the agent and represents its initial uncertainty about $\Psi_{*}$. In summary, all variables in our environment are generated as 
\begin{align}\label{eq:model}
  \Psi_* &\sim Q_{0}\,, \\
  \theta_{*, i} \mid \Psi_* &\sim P_{0, i}(\cdot \mid \Psi_*)\,, &  \forall i \in [K]\,,\nonumber \\
  Y_t \mid X_t, \theta_{*, A_t} &\sim P(\cdot \mid X_t; \theta_{*, A_t})\,, &  \forall t \in [n]\,, \nonumber
\end{align}
where $P(\cdot \mid x; \theta_{*, i})$ is the \emph{reward distribution} of action $i$ in context $x$, which only depends on parameter $\theta_{*, i}$ and the context $x$. The terminology of effect parameters arises from the fact that $\psi_{*, \ell}$ affect the model parameters $\theta_{*, i}$, which in turn define $Y_t$. The effects are mixed through the action prior $P_{0, i}$ and hence the name \emph{mixed-effect}.

Our setting can be viewed as a two-level graphical model, where $\psi_{*, 1}, \ldots, \psi_{*, L}$ are parent nodes and $\theta_{*, 1}, \ldots, \theta_{*, K}$ are child nodes (\cref{fig:setting}). The \emph{structure} is represented by missing arrows from parent (effect parameters) to child (action parameters) nodes. A missing arrow from parent $\psi_{*, \ell}$ to child $\theta_{*, i}$ means that action $i$ is independent of the $\ell$-th effect. Such models are common in offline learning, for instance QMR-DT \citep{qmr}.

All examples in \cref{sec:introduction} can be captured by our model. In movie recommendation, the categories $\ell \in [L]$ and movies $i \in [K]$ would be represented by the effect parameters $\psi_{*, \ell}$ and action parameters $\theta_{*, i}$, respectively. The weight $b_{i, \ell}$ is the relevance of movie $i$ to category $\ell$.

\begin{figure}
  \centering
  \includegraphics[scale=0.7]{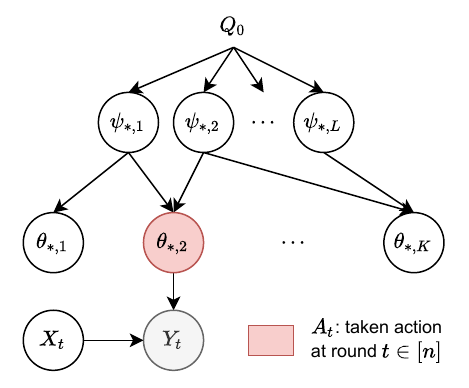}
  \caption{Example of a graphical model induced by \eqref{eq:model}.}
  \vspace{-0.1in}
  \label{fig:setting}
\end{figure}

\subsection{Notion of Optimality}
\label{sec:optimality}

Let $\Theta_* = (\theta_{*, i})_{i \in [K]} \in \mathbb{R}^{Kd}$ be the concatenation of all action parameters. The expected reward of action $i \in [K]$ in context $x \in \cX$ is $r(x, i; \Theta_*) = \E{Y \sim P(\cdot \mid x; \theta_{*, i})}{Y}$, where $r$ is the \emph{reward function}. Our setting is Bayesian and thus a natural goal for the agent is to minimize its \emph{Bayes regret} 
\begin{align*}
  \mathcal{B}\mathcal{R}(n)
  = \mathbb{E}\left[ \sum_{t = 1}^n r(X_t, A_{t, *}; \Theta_*) - r(X_t, A_t; \Theta_*)\right]\,,
\end{align*}
where $A_{t, *} = \argmax_{i \in[K]} r(X_t, i; \Theta_*)$ is the optimal action in round $t$. The above expectation is over all random variables in \eqref{eq:model}. While the Bayes regret is weaker than the frequentist regret, it is a reasonable metric for average performance across multiple instances \citep{russo14learning}. We present a special case of our setting next.

\subsection{Linearity in Effects}

A simple yet powerful assumption is that the action prior $P_{0, i}$ is parametrized by a weighted sum of effect parameters
\begin{align*}
    \theta_{*, i} &\mid \Psi_* \sim P_{0, i}\Big(\cdot \Big| \sum_{\ell =1}^L b_{i, \ell} \psi_{*, \ell}\Big)\,, & \forall i \in [K]\,,
\end{align*} 
where $b_i = (b_{i, \ell})_{\ell \in [L]} \in \mathbb{R}^L$ are $L$ \emph{known mixing weights} for action $i$. The effect $\ell$ on action $i$ is determined by $b_{i, \ell}$. As an example, $b_{i, \ell}=0$ when action $i$ is independent of effect $\ell$. This is an important special case of our setting since additive models are widely used in both theory and practice \citep{mccullagh89generalized}, as they often lead to closed-form posteriors that are computationally tractable. Next we present two instances of this setting, where $P_{0, i}$ is a multivariate Gaussian with mean $\sum_{\ell =1}^L b_{i, \ell} \psi_{*, \ell}$ and covariance $\Sigma_{0, i}$. We defer non-linear effects to \cref{subsec:non_linear_mixed_model}.

\subsection{Mixed-Effect Linear Bandit}
\label{subsec:contextual_gaussian_bandits}

A natural joint effect prior $Q_{0}$ for $d$-dimensional effect parameters $\psi_{*, \ell}$ is a multivariate Gaussian with mean $\mu_{\Psi} \in \real^{Ld}$ and covariance $\Sigma_{\Psi} \in \real^{Ld \times Ld}$. The action prior $P_{0, i}$ is a Gaussian with mean $\sum_{\ell =1}^L b_{i, \ell} \psi_{*, \ell} \in \real^d$ and covariance $\Sigma_{0, i} \in \real^{d \times d}$. This model is a variant of a linear Gaussian model \citep{koller09probabilistic} and is given by
\begin{align}\label{eq:contextual_gaussian_model}
    \Psi_{*} & \sim \cN(\mu_{\Psi}, \Sigma_{\Psi})\,, \\ 
    \theta_{*, i} \mid  \Psi_{*}& \sim \cN\Big( \sum_{\ell =1}^L b_{i, \ell} \psi_{*, \ell} , \,  \Sigma_{0, i}\Big)\,, & \forall i \in [K]\,,\nonumber\\ 
    Y_t  \mid X_t, \theta_{*, A_t} & \sim \cN(X_t^\top \theta_{*, A_t} ,  \sigma^2)\,, & \forall t \in [n]\,, \nonumber
\end{align}
where $\sigma>0$ is an observation noise. This model reduces to a multi-armed bandit (\cref{additional_posteriors}) when $X_t=1$ for all $t \in [n]$ and $\psi_{*, \ell} \in \real$ for all $\ell \in [L]$. 

\subsection{Mixed-Effect Generalized Linear Bandit}
\label{subsec:contextual_gl_bandits}

Here the effect and action parameters are generated as in \eqref{eq:contextual_gaussian_model} but the reward $Y_t$ is sampled from a \emph{generalized linear model (GLM)} \citep{mccullagh89generalized}, which is non-linear. In particular, $P(\cdot \mid X_t; \theta)$ is an exponential-family distribution with mean $f(X_t^\top \theta)$ and the whole model is
\begin{align}\label{eq:contextual_bernoulli_model}
    \Psi_{*} & \sim \cN(\mu_{\Psi}, \Sigma_{\Psi})\,,\\ 
    \theta_{*, i} \mid  \Psi_{*}& \sim \cN\Big( \sum_{\ell =1}^L b_{i, \ell} \psi_{*, \ell} , \,  \Sigma_{0, i}\Big)\,, & \forall i \in [K]\,,\nonumber\\ 
  Y_t \mid X_t, \theta_{*, A_t} &\sim P(\cdot \mid X_t; \theta_{*, A_t})\,, &  \forall t \in [n]\,.\nonumber
\end{align} 
Let $ \mathrm{Ber}(p)$ be a Bernoulli distribution with mean $p$. One particular choice of a GLM is $f(u) = 1 / (1 + \exp(-u))$ and $P(\cdot \mid X_t; \theta) = \mathrm{Ber}(f(X_t^\top \theta))$, which corresponds to a logistic bandit \citep{filippi10parametric}.

\subsection{Structure Learning}
\label{subsec:creating_structure}

As discussed in \cref{sec:introduction}, the structures in \eqref{eq:contextual_gaussian_model} and \eqref{eq:contextual_bernoulli_model} may be intrinsic in some problems, such as drug design. When this is not the case, we propose the following approach to learning a \emph{proxy structure}. For any $i \in [K]$, let $\hat{\theta}_i$ represent an offline estimate of action parameter $\theta_{i}$. To learn, we fit a Gaussian mixture model (GMM) \citep{reynolds2009gaussian} with $L$ clusters to $\hat{\theta}_i$. Each cluster $\ell \in [L]$ is represented by its center $\mu_{\psi_\ell} \in \real^d$ and covariance $\Sigma_{\psi_\ell} \in \real^{d \times d}$. These correspond to the mean of the effect parameter $\psi_\ell$ and its uncertainty. The GMM also outputs the probability that $\hat{\theta}_i$ belongs to cluster $\ell$, for all combinations of $i \in [K]$ and $\ell \in [L]$. This probability is the mixing weight $b_{i, \ell}$. 

The above procedure is general and can be adapted to any use case. The main challenge is to obtain the offline estimates $\hat{\theta}_i$. This is an offline representation learning problem \citep{tripuraneni2021provable}, for which numerous techniques exist. For instance, in our MovieLens experiments in \cref{sec:movielens experiments}, we use a low-rank factorization of the rating matrix to obtain these offline estimates. This is a strength of our approach. It is highly flexible and can be easily integrated with popular and practical offline learning tools. This can be seen as a step towards bridging the gap between offline and online learning.

%% file: algorithm.tex
\section{ALGORITHM}
\label{sec:algorithm}

\begin{algorithm}[t]
\caption{\alg: \textbf{M}ixed-\textbf{E}ffect \textbf{T}hompson \textbf{S}ampling.}
\label{alg:ts}
\textbf{Input:} Joint effect prior $Q_{0}$, action priors $P_{0, \cdot}$ \\
Initialize $Q_1 \gets Q_0$ and $P_{1, \cdot} \gets P_{0, \cdot}$ \\
\For{$t=1, \dots, n$}
{Sample $\Psi_t \sim Q_{t}$ \\
\For{$i=1, \dots, K$}{Sample $\theta_{t, i} \sim P_{t, i} (\cdot \mid  \Psi_t)$}
$\Theta_t \gets (\theta_{t, i})_{i \in [K]}$ \\
$A_t \gets \argmax_{ i \in [K]} r(X_t, i; \Theta_t)$ \\
Receive reward $Y_t \sim P(\cdot \mid X_t; \theta_{*, A_t})$ \\
Compute new posteriors $Q_{t+1}$ and $P_{t+1, \cdot}$
}
\end{algorithm}

We propose a Thompson sampling algorithm \citep{thompson33likelihood, russo14learning, ts_scott}, which is a natural Bayesian solution to our problem. The algorithm is based on hierarchical sampling \citep{lindley72bayes}, which reflects the structure in our model. Before we present it, we need to introduce additional notation. We denote by $H_t = (X_\ell, A_\ell, Y_\ell)_{\ell \in [t-1]}$ the \emph{history} of all interactions of the agent up to round $t$, by $S_{t, i} = \{\ell \in [t - 1]: A_\ell = i\}$ the rounds where the agent takes action $i$ up to round $t$, and by $H_{t, i} = (X_\ell, A_\ell, Y_\ell)_{\ell \in S_{t, i}}$ the corresponding history.

Our algorithm \alg is presented in \cref{alg:ts}. Because the effect parameters are shared by all actions, their posteriors are not independent. For this reason, we maintain a single \emph{joint effect posterior}
\begin{align*}
  Q_{t}(\Psi)
  = \condprob{\Psi_{*} = \Psi }{H_t}
\end{align*}
for all effect parameters $\Psi_*$ in round $t$. Moreover, we maintain an \emph{action posterior} 
\begin{align*}
  P_{t, i}(\theta \mid \Psi)
  = \condprob{\theta_{*, i} = \theta}{H_{t, i}, \Psi_{*} = \Psi}
\end{align*}
for each action $i \in [K]$ given $\Psi_* = \Psi$. \alg samples hierarchically as follows. In round $t$, we first sample effect parameters $\Psi_t \sim Q_t$. Then we sample each action parameter $\theta_{t, i} \sim P_{t, i}(\cdot \mid \Psi_t)$ individually. Note that this is equivalent to sampling from the exact posterior $\condprob{\theta_{*, i} = \theta}{H_t}$, since
\begin{align}\label{eq:sampling_equivalence}
  \condprob{\theta_{*, i} = \theta}{H_t} &=\int_{\Psi}  \condprob{\theta_{i, *} = \theta, \Psi_{*} = \Psi}{H_t} \dif \Psi\,,\nonumber\\ &= \int_{\Psi} P_{t, i}(\theta \mid \Psi)
  Q_t(\Psi) \dif \Psi\,.
\end{align}
Finally, we behave optimistically and take the action with the highest expected reward under the posterior-sampled action parameters $\Theta_t = (\theta_{t, i})_{i \in [K]}$.

\subsection{Posterior Derivations}

The posteriors are computed as follows. We first express the joint effect posterior $Q_t$ as
\begin{align}\label{eq:qt_derviation}
  Q_{t}(\Psi) 
  & \propto \prod_{i =1}^K \int_{\theta} \LL_{t, i}(\theta) P_{0, i}\left(\theta \mid \Psi \right) \dif \theta \ Q_0(\Psi)\,, 
\end{align}
where $\LL_{t, i}(\theta) = \condprob{H_{t, i}}{\theta_{*, i}=\theta} = \prod_{(x, a, y) \in H_{t, i}} P(y \mid x; \theta)$ is the likelihood of all observations of action $i$ up to round $t$ given $\theta_{*, i} = \theta$. Next, for any action $i \in [K]$, the action posterior $P_{t, i}$ is defined as 
\begin{align}\label{eq:pti_derviation}
    P_{t, i}(\theta \mid \Psi) 
    &\propto \LL_{t, i}(\theta) P_{0, i}(\theta \mid \Psi)\,.
\end{align} 
$P_{t, i}$ is similarly sparse to $P_{0, i}$. Specifically, in any round $t$, $P_{t, i}$ and $P_{0, i}$ are parameterized by the same subset of effect parameters $\Psi_*$, since $\LL_{t, i}(\theta)$ does not depend on $\Psi_*$.

The joint effect posterior $Q_t$ and action posteriors $P_{t, i}$ have closed forms in Gaussian models, which allows efficient sampling and theoretical analysis. Beyond these, MCMC and variational inference \citep{doucet01sequential} can be used to approximate $Q_t$ and $P_{t, i}$. Next we derive closed-form posteriors for the mixed-effect model with linear rewards in \eqref{eq:contextual_gaussian_model} and provide an efficient approximation for the mixed-effect model with non-linear rewards in \eqref{eq:contextual_bernoulli_model}.

\subsection{Mixed-Effect Linear Bandit}
\label{sec:linear bandit posterior}

Let the outer product of contexts corresponding to action $i$ up to round $t$ be $G_{t, i} = \sigma^{-2} \sum_{\ell \in S_{t, i}} X_\ell X_\ell^\top$ and their sum weighted by rewards be $B_{t, i} = \sigma^{-2} \sum_{\ell \in S_{t, i}} Y_\ell X_\ell$. Both $G_{t, i}$ and $B_{t, i}$ are scaled by the observation noise $\sigma$. Using these quantities, the effect posterior is defined as follows.

\begin{proposition}\label{thm:pt_gaussian}
For any round $t \in [n]$, the joint effect posterior is a multivariate Gaussian $Q_t = \cN(\bar{\mu}_t, \bar{\Sigma}_t)$, where 
\begin{align}    \label{eq:linear effect posterior}
    &\bar{\Sigma}_t^{-1}= \Sigma_{\Psi}^{-1} +  \sum_{i =1}^K b_i b_i^\top \otimes \left( \Sigma_{0, i} + G_{t, i}^{-1} \right)^{-1}\,,\\
    \bar{\mu}_t &= \bar{\Sigma}_t \Big(\Sigma_{\Psi}^{-1}  \mu_{\Psi} +\sum_{i =1}^K b_i \otimes ( (\Sigma_{0, i} + G_{t, i}^{-1})^{-1} G_{t, i}^{-1} B_{t, i})\Big)\,.\nonumber
\end{align}
\end{proposition}

The effect posterior is additive in individual actions and can be interpreted as follows. Each action is a single noisy observation in its estimate. The \emph{maximum likelihood estimate (MLE)} of the parameter of action $i$, $G_{t, i}^{-1} B_{t, i}$, contributes to \eqref{eq:linear effect posterior} proportionally to its precision, $(\Sigma_{0, i} + G_{t, i}^{-1})^{-1}$. The contribution to the $\ell$-th effect parameter is weighted by $b_{i, \ell}$, which is the mixture weight for $\theta_{*, i}$ in \eqref{eq:contextual_gaussian_model}. \cref{thm:pt_gaussian} is proved in \cref{app:effect posterior-derivation}.

Note that $G_{t, i}$ in \cref{thm:pt_gaussian} may not be invertible. We want to stress that the formulas with it are for the ease of exposition only. The reason is that $G_{t, i}^{-1}$ appears after using the Woodbury matrix identity to invert $\Sigma_{0, i}^{-1} + G_{t, i}$, which is well defined. Precisely, we use that
\begin{align*}
&\Sigma_{0, i}^{-1} - \Sigma_{0, i}^{-1} (G_{t, i} + \Sigma_{0, i}^{-1})^{-1}\Sigma_{0, i}^{-1}  = \big(\Sigma_{0, i} + G_{t, i}^{-1}\big)^{-1}\,, \\
&\Sigma_{0, i}^{-1} (G_{t, i} + \Sigma_{0, i}^{-1})^{-1} B_{t, i} =   (\Sigma_{0} + G_{t, i}^{-1})^{-1} G_{t, i}^{-1} B_{t, i}\,.
\end{align*}
For consistency, the formulas in \cref{thm:pt_gaussian} are also used in our experiments, where $(G_{t, i} + 10^{-3} I_d)^{-1}$ replaces $G_{t, i}^{-1}$ for numerical stability. We note that this results in a similar regret to using the correct formulas. 

Now we present the action posterior.
\begin{proposition}\label{thm:pti_gaussian}
For any round $t \in [n]$, action $i \in [K]$, and effect parameters $\Psi_t$, the action posterior is a multivariate Gaussian $P_{t, i}(\cdot \mid \Psi_t) = \cN(\cdot;\tilde{\mu}_{t, i}, \tilde{\Sigma}_{t, i})$, where
\begin{align}  \label{eq:linear action posterior}
   \tilde{\Sigma}_{t, i}^{-1}  &= \Sigma_{0, i}^{-1} + G_{t, i}\,, \\
    \tilde{\mu}_{t, i} &= \tilde{\Sigma}_{t, i} \Big( \Sigma_{0, i}^{-1} \sum_{\ell =1}^L b_{i, \ell} \psi_{t, \ell} + B_{t, i}  \Big)\,.\nonumber
\end{align}
\end{proposition}
The action posterior in \eqref{eq:linear action posterior} is a standard multivariate Gaussian posterior whose prior depends on $\Psi_t$, which is sampled by \alg. \cref{thm:pti_gaussian} is proved in \cref{app:conditional-posterior-derivation}.

\subsection{Mixed-Effect Generalized Linear Bandit}
\label{sec:glb bandit posterior}

Closed-form posteriors do not exist in this setting and approximations are needed. We opt for a simple scheme that approximates the likelihood $\LL_{t, i}(\cdot)$ by a multivariate Gaussian using the Laplace approximation. In particular, since $P(\cdot \mid X_t; \theta)$ is an exponential-family reward distribution, the log-likelihood for any action $i \in [K]$ is
\begin{align*}
  \log \LL_{t, i}(\theta)
  = \sum_{\ell \in S_{t, i}} Y_\ell X_\ell^\top \theta - A(X_\ell^\top \theta) + C(Y_\ell)\,,
\end{align*}
where $C$ denotes a real function, and $A$ is a twice continuously differentiable function whose derivative is the mean function $f$, $\dot{A} = f$. Let $\mu^\textsc{lap}_{t, i}$ and $G^\textsc{lap}_{t, i}$ be the MLE and the Hessian of $- \log \LL_{t, i}(\cdot)$, respectively, defined as 
\begin{align}
        \mu^\textsc{lap}_{t, i} &= \argmax_{\theta \in \real^d} \log \LL_{t, i}(\theta)\,,\\
      G^\textsc{lap}_{t, i} &= \sum_{\ell \in S_{t, i}} \dot{f}\left(X_\ell^\top \mu^\textsc{lap}_{t, i}\right) X_\ell X_\ell^\top\,.\nonumber
\end{align} 
Then the Laplace approximation is
\begin{align}\label{eq:laplace_approximation}
    \LL_{t, i}(\cdot) \approx \cN(\cdot; \mu^\textsc{lap}_{t, i}, (G^\textsc{lap}_{t, i})^{-1})\,.
\end{align}
Now we plug \eqref{eq:laplace_approximation} into \eqref{eq:qt_derviation} and have $Q_{t}(\cdot) \approx \cN(\cdot;\bar{\mu}_{t}, \bar{\Sigma}_{t})$, where $\bar{\mu}_{t}$ and $\bar{\Sigma}_{t}$ are computed as in \cref{thm:pt_gaussian}, except that $G^\textsc{lap}_{t, i}$ and $\mu^\textsc{lap}_{t, i}$ replace $G_{t, i}$ and $G^{-1}_{t, i}B_{t, i}$, respectively. We plug \eqref{eq:laplace_approximation} into \eqref{eq:pti_derviation} and get $P_{t, i}(\cdot \mid \Psi) \approx \cN(\cdot;\tilde{\mu}_{t, i}, \tilde{\Sigma}_{t, i})$, where $\tilde{\mu}_{t, i}$ and $\tilde{\Sigma}_{t, i}$ are computed as in \cref{thm:pti_gaussian}, except that $G^\textsc{lap}_{t, i}$ and $G^\textsc{lap}_{t, i} \mu^\textsc{lap}_{t, i}$ replace $G_{t, i}$ and $B_{t, i}$, respectively. While these results are a direct consequence of plugging the Laplace approximation \eqref{eq:laplace_approximation} into \eqref{eq:qt_derviation} and \eqref{eq:pti_derviation}, there is a clear intuition behind them. First, $G_{t, i} \gets G^\textsc{lap}_{t, i}$ captures the change of curvature due to the non-linearity of the mean function $f$. Moreover, $G^{-1}_{t, i}B_{t, i} \gets \mu^\textsc{lap}_{t, i}$ follows from the fact that the MLE of the action parameter $\theta_{*, i}$ in the linear case (\cref{subsec:contextual_gaussian_bandits}) is $G_{t, i}^{-1} B_{t, i}$, and it corresponds to $\mu^\textsc{lap}_{t, i}$ in the generalized linear case. As discussed before, $G^\textsc{lap}_{t, i}$ may not be invertible. Therefore, we approximate its inverse in our experiments by $(G^\textsc{lap}_{t, i} + 10^{-3} I_d)^{-1}$.

\subsection{Computational Complexity}
\label{sec:alternative algorithm designs}

The Bayes regret in \cref{sec:optimality} does not directly depend on $\Psi_*$. Thus the benefit of modeling the effect parameters is not immediately clear. It is tempting to marginalize them out, and only maintain a single joint posterior of all action parameters $\Theta_* \in \real^{Kd}$. Although this is feasible, posterior updates would be complex and computationally inefficient when $K \gg L$, which is common in practice.

The main advantage of \alg is that the sampling of effect parameters $\Psi_t \sim Q_t$ allows us to use the conditional independence of actions given $\Psi_*$, and model $\theta_{*,i} \mid  H_{t, i}, \Psi_* = \Psi_t$ individually. This is more computationally efficient than modeling $\Theta_* \mid H_t$ when $K \gg L$. To see this, suppose that all posteriors are multivariate Gaussians (\cref{sec:linear bandit posterior}). In this case, $\Theta_* \mid H_t$ requires $\mathcal{O}(K^2 d^2)$ space, due to storing a $Kd \times Kd$ covariance matrix; while \alg requires only $\mathcal{O}((L^2 + K) d^2)$ space, due to storing the covariances of $Q_t$ and $P_{t, i}$. Since the sampling relies on covariance inverses, the time complexity also improves. For the joint posterior, it is $\mathcal{O}(K^3 d^3)$, while it is only $\mathcal{O}((L^3 + K) d^3)$ for \alg.

One can also marginalize out the effect parameters $\Psi_*$ and have $K$ separate posteriors, one for each action parameter $\theta_{*, i}$. While this improves computational efficiency, it does not model that the actions are correlated, since $\theta_{*, i} \mid H_{t, i}$ is modeled instead of $\theta_{*, i} \mid H_{t}$. This leads to a statistical inefficiency due to the loss of information as the histories of other actions $H_{t, j}$ are discarded. We validate this through theory (\cref{sec:benefits_structure}) and experiments (\cref{sec:experiments}).

%% file: analysis.tex
\section{ANALYSIS}
\label{sec:analysis}

This section is organized as follows. First, we state our regret bound. Second, we discuss how it captures the structure of our problem. Finally, we sketch its proof. We use $\tilde{\mathcal{O}}$ for the big O notation up to polylogarithmic factors.

\subsection{Main Result}
\label{sub:main_result}

We analyze \alg in the linear setting in \cref{subsec:contextual_gaussian_bandits}. To ease exposition, we assume that there exist $\sigma_{0}, \sigma_{\Psi}, \kappa_{x} > 0$ such that $\Sigma_{0, i} = \sigma_{0}^2 I_d$ for all $i \in [K]\,,$ $\Sigma_{\Psi} = \sigma_{\Psi}^2 I_{Ld}\,,$ and $\normw{X_t}{2}^2 \leq \kappa_x$ for all $t \in [n]$. The last assumption is standard and we relax the remaining two in \cref{proof:regret_proof}.

\begin{theorem}
\label{thm:regret} For any $\delta \in (0, 1)$, the Bayes regret of \emph{\alg} in the mixed-effect model in \cref{subsec:contextual_gaussian_bandits} is bounded as
\begin{align}\label{eq:regret_terms}
  \mathcal{B}\mathcal{R}(n)
  \leq \sqrt{2 n 
  \left( \mathcal{R}^{\textsc{a}}(n) + \mathcal{R}^{\textsc{e}}(n) \right) \log(1 / \delta)} +
  c n\delta\,,
\end{align}
where $c = \sqrt{\frac{2}{\pi} \kappa_x(\sigma_0^2 + \kappa_b \sigma_\Psi^2 )}K\,, \ \kappa_b = \max_{i \in [K]}\normw{b_i}{2}^2\,,$
\begin{talign*}
& \mathcal{R}^{\textsc{a}}(n) = dK c_\textsc{a} \log\big(1 + \frac{n\kappa_x\sigma_0^2}{d \sigma^2} \big)\,, \, c_\textsc{a} = \frac{ \kappa_x\sigma^2_{0}}{\log\big(1 + \frac{\kappa_x\sigma_0^2}{\sigma^{2}}\big)}\,,\\ 
&\mathcal{R}^{\textsc{e}}(n) = dL c_\textsc{e} \log\big(1 +  \frac{K \kappa_b \sigma^2_{\Psi} }{\sigma^2_{0} + \frac{\sigma^2}{n \kappa_x}}\big)\,,  \, c_\textsc{e}= \frac{\kappa_x \kappa_b \sigma_\Psi^2 \big(1 + \frac{\kappa_x\sigma_0^2}{\sigma^{2}}  \big)}{\log\big(1 + \frac{\kappa_x \kappa_b \sigma_\Psi^2}{\sigma^{2}}\big)}\,.&
\end{talign*}
\end{theorem}
The second term in \eqref{eq:regret_terms} is constant for $\delta = 1 / n$, in which case the above bound is $\tilde{\mathcal{O}}(\sqrt{n})$ and optimal in the horizon $n$. The main quantities of interest are $\mathcal{R}^{\textsc{a}}(n)$ and $\mathcal{R}^{\textsc{e}}(n)$, and they have natural interpretations. $\mathcal{R}^{\textsc{a}}(n)$ corresponds to the action regression problem: with $K$ parameters of dimension $d$, prior width $\sigma_{0}$, maximum context length $\sqrt{\kappa_x}$, and $n$ observations with noise $\sigma$. The dependence of $\mathcal{R}^{\textsc{a}}(n)$ on these quantities is identical to a corresponding linear bandit \citep{lu19informationtheoretic}. On the other hand, $\mathcal{R}^{\textsc{e}}(n)$ corresponds to the effect regression problem: with $L$ parameters of dimension $d$, prior width $\sigma_\Psi$, maximum mixing-weight length $\sqrt{\kappa_b}$, and $K$ actions that can be viewed as observations with noise $\sigma_{0}$ (\cref{sec:linear bandit posterior}). The dependence of $\mathcal{R}^{\textsc{e}}(n)$ on these quantities mimics those in $\mathcal{R}^{\textsc{a}}(n)$.

To simplify exposition, let $\kappa_x = \kappa_b = \sigma = 1$. Then 
\begin{align}\label{eq:simplified_regret}
  \mathcal{B}\mathcal{R}(n)
  = \tilde{\mathcal{O}}\Big(\sqrt{n d (K \sigma_0^2 + L \sigma_\Psi^2(1+\sigma_0^2))}\Big)\,.
\end{align}
Note that $\mathcal{B}\mathcal{R}(n)$ decreases when the initial uncertainties $\sigma_0$ and $\sigma_\Psi$ are lower. Also smaller $K$, $L$, or $d$ mean fewer parameters to learn and lead to a lower regret. We observe these trends empirically in \cref{app:experiments}.







Our analysis is under the assumption that the covariances and mixing weights are known. This is typical in Bayesian analyses and represents prior knowledge that reduces regret. Parameter misspecification can be analyzed similarly to \citet{simchowitz21bayesian}. Roughly speaking, if it was $\tilde{\mathcal{O}}(1 / n^\alpha)$, the additional regret would be $\tilde{\mathcal{O}}(n^{1 - \alpha})$; and thus $\tilde{\mathcal{O}}(\sqrt{n})$ when $\alpha = 0.5$. 

\subsection{Benefits of Structure}
\label{sec:benefits_structure}

Note that we do not provide a matching lower bound. The only Bayesian lower bound that we know of is $\mathcal{O}(\log^2 n)$ for a $K$-armed bandit (Theorem 3 of \citet{lai87adaptive}). Seminal works on Bayes regret minimization \citep{russo14learning,russo16information} do not match it. Therefore, to argue that our bound reflects the problem structure, we compare \alg to agents that have access to more information or use less structure. We start with those with more information. Take \alg with known effect parameters $\Psi_*$. Then $\sigma_\Psi = 0$ in \eqref{eq:simplified_regret} and we obtain a lower regret $\mathcal{B}\mathcal{R}(n) = \tilde{\mathcal{O}}(\sqrt{n d K \sigma_0^2})$ that does not depend on $L$. Similarly, take \alg with a perfect linear model, $\theta_{*, i}= \sum_{\ell \in [L]}b_{i, \ell}\psi_{*, \ell}$ for all $i \in [K]$. Then $\sigma_0 = 0$ in \eqref{eq:simplified_regret} and we get a lower regret $\mathcal{B}\mathcal{R}(n) = \tilde{\mathcal{O}}(\sqrt{n d L \sigma_\Psi^2})$ that does not depend on $K$. This shows that $K$ in our bound arises due to modeling the stochasticity of action parameters with respect to the effect parameters, incorporated in $\Sigma_{0, i}$.

Next we consider an agent that does not know $\Psi_*$ and also does not model it. Here only $\Theta_*$ is learned (\cref{sec:alternative algorithm designs}) and this is achieved by marginalizing out $\Psi_*$ in \eqref{eq:contextual_gaussian_model} as
\begin{talign*}
 \theta_{*, i} & \sim \cN\Big( \sum_{\ell =1}^L b_{i, \ell} \mu_{\psi_\ell} , \,  \breve{\Sigma}_{0, i}\Big)\,, & \forall i \in [K]\,,\\ 
    Y_t  \mid X_t, \theta_{*, A_t} & \sim \cN(X_t^\top \theta_{*, A_t} ,  \sigma^2)\,, & \forall t \in [n]\,,
\end{talign*}
where $\breve{\Sigma}_{0, i} = (\sigma_0^2 + \normw{b_i}{2}^2 \sigma_\Psi^2 )I_d $ is the marginal prior covariance of action $i$ and $\mu_{\psi_\ell} \in \real^d$ is the prior mean of the $\ell$-th effect, which satisfies $\mu_\Psi = (\mu_{\psi_\ell})_{\ell \in [L]}$ (\cref{subsec:contextual_gaussian_bandits}). The marginal prior covariance $\breve{\Sigma}_{0, i}$ accounts for the uncertainty of the not-modeled $\Psi_*$ weighted by $\normw{b_i}{2}^2$. The regret of this agent scales as in \eqref{eq:simplified_regret} with $\sigma_\Psi = 0$, except that the maximum prior variance $\sigma_{0}^2$ is replaced with the maximum marginal prior variance $\sigma_0^2 + \sigma_\Psi^2$. This can be proved using the definition of $\breve{\Sigma}_{0, i}$ and $\kappa_b=\max_{i \in [K]} \normw{b_i}{2}^2 = 1$. The resulting regret bound is $\mathcal{B}\mathcal{R}(n) = \tilde{\mathcal{O}}(\sqrt{n d K (\sigma_0^2 + \sigma_\Psi^2)})$. When $K > L$ up to constants, it can be significantly higher than the regret bound of \alg in \eqref{eq:simplified_regret}. The improvement is on the order of $\sqrt{K / L}$ when the effect parameters are much more uncertain than the action ones, $\sigma_\Psi \gg \sigma_0$, which is expected. For instance, in our ad placement example, $L$ is the number of items in the catalog and $K\approx L^M$ is the number of slates of size $M$. Hence $K/L \approx L^{M-1}$, where we can have $L \approx 10^6$ and $M \approx 10$. This claim is also supported empirically in \cref{sec:synthetic experiments,app:experiments}, where \alg outperforms classic methods when the effect parameters are more uncertain than the action ones.

\subsection{Sketch of the Regret Proof}
\label{sub:sketch}
Now we outline the key technical challenges and novel insights in our proof. Our hierarchical sampling is equivalent to sampling from the exact posterior, which is also a multivariate Gaussian $\condprob{\theta_{*,i} = \theta}{H_t} = \cN(\theta; \hat{\mu}_{t,i}, \hat{\Sigma}_{t,i})$ for some $\hat{\mu}_{t,i}$ and $\hat{\Sigma}_{t,i}$. This is because the action posterior $P_{t,i}$ and effect posterior $Q_t$ are Gaussians, and Gaussianity is preserved after marginalization \citep{koller09probabilistic}. Next notice that the context $X_t$ in round $t$ is known, and thus we include it in the history $H_t$. Now let $\bm{A}_{t} \in \{0, 1\}^K$ and $\bm{A}_{t, *} \in \{0, 1\}^K$ be the indicator vectors of the taken action $A_t$ and optimal action $A_{t, *}$, respectively. Moreover, let $\hat{\theta}_t = (X_t^\top\hat{\mu}_{t, i})_{i \in [K]}$ and $\theta_{t,*} = (X_t^\top \theta_{*, i})_{i \in [K]}$. Then the Bayes regret can be decomposed following \citet{russo14learning} as 
\begin{sizeddisplay}{\small}
\begin{align}\label{standard_decomposition}
 &   \mathcal{B}\mathcal{R}(n)
  = \mathbb{E}\Big[ \sum_{t = 1}^n \bm{A}_{t, *}^\top \theta_{t,*} - \bm{A}_{t}^\top \theta_{t,*}\Big]\,,\\
 &= \E{}{\condE{\bm{A}_{t, *}^\top (\theta_{t, *} - \hat{\theta}_t)}{H_t}} +
  \E{}{\condE{\bm{A}_t^\top (\hat{\theta}_t - \theta_{t, *})}{H_t}}.\nonumber
\end{align}\end{sizeddisplay}
The identity \eqref{standard_decomposition} holds because $\hat{\theta}_t$ is deterministic given $H_t$ (which now includes $X_t$), and the actions $\bm{A}_{t, *}$ and $\bm{A}_t$ are i.i.d. given $H_t$. Conditioned on $H_t$, $\hat{\theta}_t - \theta_{t,*}$ is a zero-mean Gaussian random vector independent of $\bm{A}_t$, and therefore $\mathbb{E}[\bm{A}_t^\top (\hat{\theta}_t - \theta_{t, *})\mid H_t]=0$. Thus we only need to bound the first term in \eqref{standard_decomposition}, which can be further bounded as
\begin{sizeddisplay}{\small}
\begin{align}\label{standard_regret_bound}
   \mathcal{B}\mathcal{R}(n) &\leq  \sqrt{2 n \log(1/\delta)}
  \sqrt{\E{}{\sum_{t = 1}^n \normw{X_t}{\hat{\Sigma}_{t,A_t}}^2}} +
  c n\delta\,.
\end{align} 
\end{sizeddisplay}
To bound \eqref{standard_regret_bound}, we need to bound a $\hat{\Sigma}_{t, A_t}$-norm, while we only know closed forms of $\bar{\Sigma}_t$ and $\tilde{\Sigma}_{t, A_t}$ (\cref{sec:linear bandit posterior}). We relate these norms by generalizing the total covariance decomposition in \citet{hong22hierarchical} to incorporate multiple effect parameters and their mixing weights $b_{i, \ell}$. To account for multiple effects, we represent all effects using a single $Ld$-dimensional vector $\Psi_* = (\psi_{*, \ell})_{\ell \in [L]}$. Moreover, we let $\Gamma_i = b_i^\top \otimes I_d \in \real^{d \times Ld}$ and observe that
\begin{talign}\label{compact_gamma_i}
    \sum_{\ell =1}^L b_{i, \ell} \psi_{*, \ell} &= \Gamma_i \Psi_*\,, & \forall i \in [K]\,.
\end{talign}
The key insight here is that we rewrite $\sum_{\ell =1}^L b_{i, \ell} \psi_{*, \ell}$ as a linear function of $\Psi_*$, $\Gamma_i \Psi_*$, and encode the dependencies between the action parameter $\theta_{*, i}$ and \emph{all} effect parameters $\Psi_*$ in $\Gamma_i$. This reformulation allows us to extend the total covariance decomposition (\cref{lemma:gaussian_covariance} in \cref{subsec:regre_proof_2}) as
\begin{align}\label{cov_matrix_marginal}
\hat{\Sigma}_{t,i} = \tilde{\Sigma}_{t,i}  + \tilde{\Sigma}_{t,i}  \Sigma_{0,i}^{-1}  \Gamma_i \bar{\Sigma}_t \Gamma_i^\top \Sigma_0^{-1}  \tilde{\Sigma}_{t,i} \,, \, \forall i \in [K]\,.
\end{align}
The first term in \eqref{cov_matrix_marginal} captures uncertainty in $\theta_{*, i} \mid \Psi_*$. The second captures uncertainty in $\Psi_*$, weighted by $\tilde{\Sigma}_{t,i}$, $\Sigma_{0,i}$, and the mixing weights $\Gamma_i$ for action $i$. The term $\Gamma_i \bar{\Sigma}_t \Gamma_i^\top$ is controlled using the maximum eigenvalue of $\Gamma_i \Gamma_i^\top$, which is at most $\max_{i \in [K]}\normw{b_i}{2}^2$. Finally, we use the identities in \eqref{compact_gamma_i} and \eqref{cov_matrix_marginal} to bound the $\hat{\Sigma}_{t, A_t}$-norm in \eqref{standard_regret_bound}. By careful analysis, our regret bound still reflects the structure and captures potential sparsity.

%% file: experiments.tex
\section{EXPERIMENTS}
\label{sec:experiments}

\begin{figure}
  \centering
  \includegraphics[width=\linewidth]{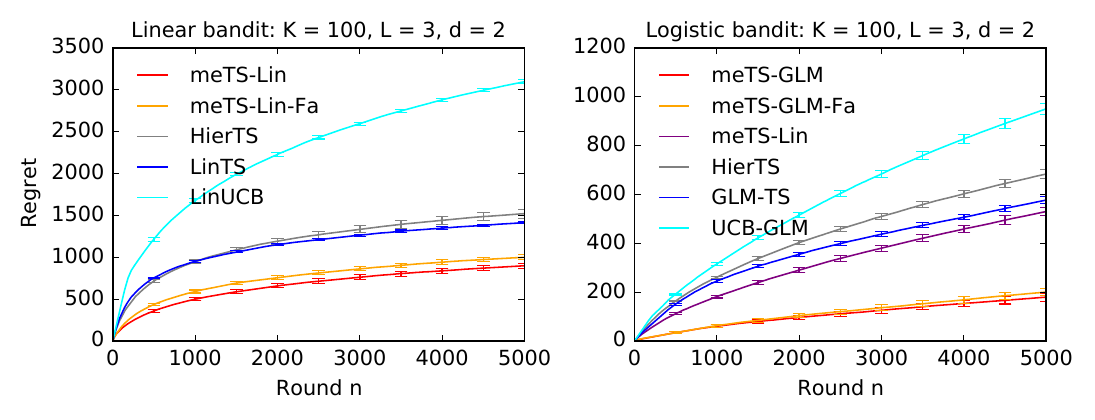}
  \vspace{-0.8cm}
  \caption{Evaluation on synthetic problems.}
  \label{fig:synthetic_regret}
\end{figure}

We evaluate \alg on both synthetic and real-world problems. In each plot, we report the average values and their standard errors. Additional experiments are conducted in \cref{app:experiments}. The code is provided in this \href{https://github.com/imadaouali/Mixed-Effect-Thompson-Sampling}{Github repository}.

\subsection{Synthetic Experiments}
\label{sec:synthetic experiments}

We start with two synthetic problems: the linear and logistic bandit settings in \eqref{eq:contextual_gaussian_model} and \eqref{eq:contextual_bernoulli_model}, respectively. The effect prior is parameterized by $\mu_\Psi=\mathbf{0}_{Ld}$ and $\Sigma_\Psi=3I_{Ld}$, the action covariance is $\Sigma_{0,i} = I_{d}$ for all $i \in [K]$, and the observation noise is $\sigma=1$. We use this setting since modeling of the effect parameters is the most beneficial when they are more uncertain than the action ones (\cref{sec:benefits_structure}). The context $X_t$ is sampled uniformly from $[-1, 1]^d$. We run $50$ simulations and sample the mixing weights $b_{i, \ell}$ from $[-1,1]$ in each run.

We consider the following baselines. For the linear setting, we compare \alglin (\cref{sec:linear bandit posterior}), \linucb \citep{abbasi-yadkori11improved}, \lints \citep{agrawal13thompson} and \hierts \citep{hong22hierarchical}. For the logistic setting, we compare \algglm (\cref{sec:glb bandit posterior}), \alglin (\cref{sec:linear bandit posterior}), \ucbglm \citep{li17provably}, \glmts \citep{chapelle11empirical} and \hierts \citep{hong22hierarchical}. \glmucb \citep{filippi10parametric} is not included because it has a very high regret. We also include factored approximations of \alg (\alglinfa and \algglmfa), where the effect parameters are sampled individually (\cref{app:posterior_approximation}). This improves the time and space complexities of \alg by $L^2$ and $L$, respectively.

All baselines but \hierts ignore the structure. \hierts incorporates the structure similarly to \alglin but only has a single effect parameter with prior $\cN(\mathbf{0}_d,3I_{d})$, with the same mean and covariance as the effect parameters of \alg. To compare fairly with \lints and \glmts, their marginal prior mean and covariance are chosen as $\mathbf{0}_d$ and $\breve{\Sigma}_{0, i} = \Sigma_{0,i} + \Gamma_i \Sigma_\Psi \Gamma_i^\top$, where $\Gamma_i = b_i^\top \otimes I_d$. This is to account for the uncertainty of the effect parameters despite marginalizing them out.

In \cref{fig:synthetic_regret}, we plot the regret in both problems for $n=5000, K=100, L=3$, and $d=2$. \alg and its factored variant outperform all baselines that ignore the structure or incorporate it partially. Moreover, \algglm outperforms \alglin in the logistic bandit, which shows the benefit of the approximation in \cref{sec:glb bandit posterior}. This attests to the generality and flexibility of \alg and the posterior derivations in \cref{sec:algorithm}. We also show in \cref{app:synthetic} that a higher $K, L$, or $d$ leads to a higher regret due to learning more parameters, which is captured by our regret bounds.

\subsection{MovieLens Experiments}
\label{sec:movielens experiments}

We study the problem of movie recommendation using the MovieLens 1M dataset \citep{movielens}. This dataset contains one million ratings given by $6040$ users to $3952$ movies. We apply low-rank factorization to the rating matrix to obtain $5$-dimensional representations: $x_j \in \real^5$ for user $j \in [6040]$ and $\theta_i \in \real^5$ for movie $i \in [3952]$. We use the movies as actions and the context $X_t$ is sampled uniformly from user vectors $x_j$. We consider both linear and logistic rewards. Given a user $x_j$, the linear reward for movie $\theta_i$ is sampled from $\cN(x_j^\top \theta_i, \sigma^2)$ while the logistic reward is sampled from ${\rm Ber}(f(x_j^\top \theta_i))$, where $f$ is the sigmoid function. We run $50$ simulations with $K=100$ randomly sampled movies in each run. We compare \alg to most baselines in \cref{sec:synthetic experiments}. We do not include \ucbglm and \glmucb because their regret is very high. In \lints and \glmts, the prior mean of action $i$ is $\mu$ and its covariance is $\breve{\Sigma}_{0} = \mathrm{diag}(v) \in \real^{d \times d}$, where $\mu \in \real^d$ and $v \in \real^d$ are the mean and variance of the movie vectors along all dimensions, respectively.

The mixed-effect structure in \eqref{eq:contextual_gaussian_model} and \eqref{eq:contextual_bernoulli_model} is not available in this problem. Therefore, we use the approach in \cref{subsec:creating_structure} to learn it. More precisely, we cluster the movies into $L=5$ mixture components by training a GMM on the offline action vectors $\theta_i$ (\cref{subsec:creating_structure}). Each cluster center corresponds to an effect parameter mean $\mu_{\psi_{\ell}} \in \real^d$ and the mixing weight $b_{i, \ell}$ is the probability that movie $i$ belongs to cluster $\ell$, as given by the GMM. We set the effect prior covariance as $\Sigma_\Psi = 0.75 \, \mathrm{diag}((\breve{\Sigma}_{0})_{\ell \in [L]}) \in \real^{Ld \times Ld}$ and the prior covariance of action $i$ as $\Sigma_{0, i} = 0.25 \, \breve{\Sigma}_{0} \in \real^{d \times d}$, where $\breve{\Sigma}_{0}$ is the same as in both \lints and \glmts. This means that the marginal covariance of action $i$ in \alg is $0.25 \, \breve{\Sigma}_{0} + 0.75 \, \Gamma_i \Sigma_\Psi \Gamma_i^\top,$ where $\Gamma_i = b_i^\top \otimes I_d$. Therefore, it is on the same order as $\breve{\Sigma}_{0, i}$ when $\normw{b_i}{2}^2 \approx 1$, and \alg is parameterized comparably to \texttt{LinTS} and \texttt{GLM-TS}. At the same time, we also model that the effect parameters are more uncertain than the action ones, since $\Sigma_{0, i} = 0.25 \, \breve{\Sigma}_{0}$ while $\Sigma_\Psi = 0.75 \, \mathrm{diag}((\breve{\Sigma}_{0})_{\ell \in [L]})$. 

In \cref{fig:movielens_regret}, we plot the regret for $n=5000$ rounds. We observe that \alg has the lowest regret, even if the true rewards are not generated from a mixed-effect model. This shows the robustness of \alg to model misspecification, which we further validate in \cref{app:robustness_model_misspecification}. It also highlights the flexibility of our framework, where a proxy structure is learned from offline data.

\begin{figure}
\includegraphics[width=\linewidth]{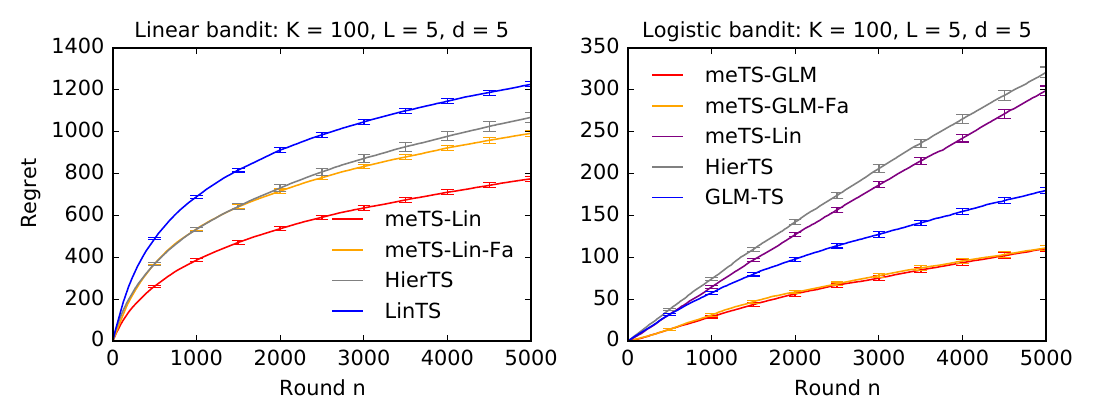}
\vspace{-0.8cm}
\caption{Evaluation on MovieLens problems.}
\label{fig:movielens_regret}
\end{figure}

%% file: related_work.tex
\section{RELATED WORK}
\label{sec:related work}

Thompson sampling (TS) \citep{thompson33likelihood} is a popular exploration algorithm in practice \citep{chapelle11empirical,russo18tutorial}. Its first Bayes regret bound was proved by \citet{russo14learning}. We apply TS to two-level graphical models with multiple parents. Many recent works \citep{bastani19meta,kveton21metathompson,basu21noregrets,simchowitz21bayesian,wan21metadatabased,hong22hierarchical,peleg22metalearning,wan22towards,tomkins2021intelligentpooling} applied TS to the two-level models with a single parent. The main difference in our work is that we consider a mixed-effect model with multiple parents in the contextual setting. \citet{urteaga18variational} proposed TS with a mixture reward distribution. This is very different from the parameter mixtures in our work.

Our analysis (\cref{sub:sketch}) extends \citet{hong22hierarchical} to multiple effects in the contextual bandit setting. The main technical challenges are generalizing the regret decomposition in \citet{hong22hierarchical} to include context and extending their total covariance decomposition to multiple effect parameters with mixing weights. This extension is non-trivial since \citet{hong22hierarchical} assumed that all action parameters are centered at a single variable. In our setting, this is not true even if we treated all effect parameters as a single vector. Moreover, the action parameters can depend on a small subset of effect parameters, resulting in sparsity that is not captured by their analysis. Although information theory can be used to derive Bayes regret bounds \citep{russo16information,lu19informationtheoretic}, we are unaware of any for multiple effects.

We also assume that there exists an underlying structure among the actions. Many such structures have been studied and we review some below. In latent bandits \citep{maillard14latent,hong20latent}, a single latent variable indexes multiple candidate models. In structured finite-armed bandits \citep{lattimore14bounded,gupta18unified}, each arm is associated with a known mean function. The mean functions are parameterized by a shared latent parameter, which is learned. TS was also applied to more complex models, such as graphical models \citep{yu20graphical} and a discretized parameter space \citep{gopalan14thompson}. While these frameworks are general, the computational and statistical efficiencies are not guaranteed simultaneously. Meta- and mutli-task learning with UCBs have a long history in bandits \citep{azar13sequential,gentile14online,deshmukh17multitask,cella20metalearning}. These works are frequentist, analyze a stronger notion of regret, and often lead to conservative algorithms. In contrast, our approach is Bayesian, we analyze its Bayes regret, and \alg performs well as analyzed without any additional tuning.

Our work is also related to representation learning in multi-task linear bandits \citep{cella2022multi,hu2021near,yang2020impact}. We refer to these works collectively as \emph{representation learning bandits}. Representation learning bandits can be viewed in our notation as learning $\Theta_* = \Psi_* \Gamma$, where $\Theta_* \in \mathbb{R}^{d \times K}$ is a matrix of action parameters, $\Psi_* \in \mathbb{R}^{d \times L}$ is a matrix of effect parameters, and $\Gamma \in \mathbb{R}^{L \times K}$ is a matrix of mixing weights. Therefore, representation learning bandits assume that the action parameters are a perfect linear combination of effect parameters, $\Sigma_{0, i} = \mathbf{0}_{d, d}$, where $\mathbf{0}_{d, d}$ denotes a $d \times d$ zero matrix. This shows that our setting is more general, since we consider $\Sigma_{0,i} \neq \mathbf{0}_{d, d}$ due to action parameter uncertainty. Consequently, representation learning bandits can have linear regret when $\Sigma_{0,i} \neq \mathbf{0}_{d, d}$, due to model misspecification.

On the other hand, representation learning bandits learn $\Gamma$ while we assume that it is given. So they can be viewed as more general. Note that the factorization $\Theta_* = \Psi_* \Gamma$ is only beneficial when $d \gg L$. In this case, $\Psi_* \Gamma$ has $d L + K L$ parameters while $\Theta_*$ would have $d K \gg d L + K L$. We do not assume that $d \gg L$. In fact, in all of our experiments, $d L + K L > d K$ and thus learning of $\Theta_*$ directly is more practical. This is what \texttt{Lin-TS} and \texttt{GLM-TS} already do, and \alg outperforms them in all experiments. Thus representation learning bandits would not be competitive in our setting. This highlights another major difference from representation learning bandits: their $L$ is the number of latent dimensions while ours is the number of effects or clusters. These are two different approaches to modeling, although coinciding algebraically when $\Sigma_{0, i} = \mathbf{0}_{d, d}$.

\alg can be extended to unknown mixing weights. However, this would require solving a matrix factorization problem online, which is expensive and representation learning bandits suffer from the same computational challenge. Since our goal is to design practical and efficient algorithms, we focus on known mixing weights. They are either given or learned offline using off-the-shelf techniques (\cref{subsec:creating_structure}).

%% file: conclusions.tex
\section{CONCLUSIONS}
\label{sec:conclusions}

We propose a mixed-effect bandit framework represented by a two-level graphical model where actions can depend on multiple effects. This structure can be used to explore more efficiently, and we design an exploration algorithm \alg that leverages it. \alg performs well on both synthetic and real-world problems when implemented as analyzed. Our algorithmic ideas apply to the general mixed-effect model in \eqref{eq:model}, although we focus on models in \eqref{eq:contextual_gaussian_model} and \eqref{eq:contextual_bernoulli_model}. Our algorithmic and theory foundations open the door to richer models, which we discuss in great detail in \cref{app:extensions}.

%% file: appendix.tex
\onecolumn
\aistatstitle{Mixed-Effect Thompson Sampling: \\
Supplementary Materials}

\section*{ORGANIZATION}
The supplementary material is organized as follows. In \textbf{\cref{app:preliminaries}}, we include a more visual notation and present some preliminary results that we use in our analysis. In \textbf{\cref{additional_posteriors}}, we introduce the mixed-effect multi-armed bandit setting and provide a closed-form solution for the corresponding effect posterior and action posteriors. In \textbf{\cref{app:posterior_derivations}}, we give the derivations of the effect posterior and action posteriors for the mixed-effect linear bandit setting. These proofs can be easily extended to the generalized linear case. In \textbf{\cref{proof:regret_proof}}, we prove an upper bound for Bayes regret of \alg. In \textbf{\cref{app:extensions}}, we discuss in details possible extensions of this work. In \textbf{\cref{app:experiments}}, we present additional experiments.

\section{PRELIMINARIES}\label{app:preliminaries}
In this section, we include additional notation and provide some basic properties of matrix operations.
\subsection{Notation}\label{app:notation}
For any positive integer $n$, we define $[n] = \{1,2,...,n\}$. We use $I_d$ to denote the identity matrix of dimension $d \times d$. Unless specified, the $i$-th coordinate of a vector $v$ is $v_i$. When the vector is already indexed, such as $v_j$, we write $v_{j, i}$. Similarly, the $(i, j)$-th entry of a matrix $\Alpha$ is $\Alpha_{i, j}$. Let $a_1 \in \real^d, \ldots, a_n \in \real^d$ be $n$ vectors. We use $\Alpha = [a_1, a_2, \ldots, a_n] \in \mathbb{R}^{d \times n}$ to denote the $d \times n$ matrix obtained by horizontal concatenation of vectors $a_1, \ldots, a_n$ such that the $j$-th column of $\Alpha$ is $a_j$ and its $(i, j)$-th entry is $A_{i, j} = a_{j, i}$. We also denote by $a = (a_i)_{i \in [n]} \in \real^{nd}$ a vector of length $nd$ obtained by concatenation of vectors $a_1, \ldots, a_n$. $\operatorname{Vec}(\cdot)$ denotes the vectorization operator. For instance, we have that $\operatorname{Vec}([a_1, \ldots, a_n]) = (a_i)_{i \in [n]}$. For any matrix $\Alpha \in \mathbb{R}^{d \times d}$, we use $\lambda_1(\Alpha)$ and $\lambda_d(\Alpha)$ to denote the maximum and minimum eigenvalue of $\Alpha$, respectively. Let $\Alpha_1, \ldots, \Alpha_n$ be $n$ matrices of dimension $d \times d$. Then $\mathrm{diag}((\Alpha_i)_{i \in [n]}) \in \real^{nd \times nd}$ denotes the block diagonal matrix where $\Alpha_1, \ldots, \Alpha_n$ are the main-diagonal blocks. Similarly, $(\Alpha_i)_{i \in [n]} \in \real^{nd \times d}$ is the $nd \times d$ matrix obtained by concatenation of $\Alpha_1, \ldots, \Alpha_n$. We use $\otimes$ to denote the Kronecker product. Now we provide a more visual presentation of the notation above. Let $a_1 \in \real^d, \ldots, a_n \in \real^d$ be $n$ vectors of dimension $d$, and let $\Alpha_1 \in \real^{d \times d}, \ldots, \Alpha_n \in \real^{d \times d}$ be $n$ matrices of dimension $d \times d$. We have that
\begin{align*}
    [a_1, \ldots, a_n] &= {\begin{pmatrix} \mid & \mid & &\mid\\a_1& a_2& \cdots &a_n\\\mid &\mid & &\mid \end{pmatrix}} \in \real^{d \times n}\,, \qquad 
    (a_i)_{i \in [n]} = {\begin{pmatrix} a_1 \\ a_2 \\ \vdots \\ a_n \end{pmatrix}} \in \real^{nd}\,,\\
    \mathrm{diag}((\Alpha_i)_{i \in [n]}) &= {\begin{pmatrix}\Alpha _{1}&0&\cdots &0\\0&\Alpha _{2}&\cdots &0\\\vdots &\vdots &\ddots &\vdots \\0&0&\cdots &\Alpha _{n}\end{pmatrix}} \in \real^{nd \times nd}\,, \qquad 
    (\Alpha_i)_{i \in [n]} = {\begin{pmatrix}  \Alpha_1 \\ \Alpha_2 \\ \vdots \\ \Alpha_n \end{pmatrix}} \in \real^{nd \times d}\,.
\end{align*}
Finally, let $A_{i, j} \in \real^{d \times d}$, for $i \in [n] \text{ and } j \in [m]$ be $nm$ matrices of dimensions $d \times d$. We use $(\Alpha_{i, j})_{(i, j) \in [n] \times [m]}$ to denote the $nd \times md$ block matrix where $\Alpha_{i, j}$ is the $(i, j)$-th block. We also provide a more visual presentation for this notation.
\begin{align*}
    (\Alpha_{i, j})_{(i, j) \in [n] \times [m]} &= {\begin{pmatrix}\Alpha_{1, 1}&\Alpha_{1, 2}&\cdots & \Alpha_{1, m}\\\Alpha_{2, 1}&\Alpha _{2, 2}&\cdots &\Alpha_{2, m}\\\vdots &\vdots &\ddots &\vdots \\\Alpha_{n, 1}&\Alpha_{n, 1}&\cdots &\Alpha _{n, m}\end{pmatrix}} \in \real^{nd \times md}\,.
\end{align*}

\newpage

\subsection{Preliminary Results}\label{sub:preliminary_results}
In this section, we recall some basic properties of matrix operations.

\begin{enumerate}[label=(\alph*)]
    \item \textbf{The mixed-product property.} We have that $(\Alpha \otimes \Beta) (\mathrm{C} \otimes \mathrm{D}) = \Alpha \mathrm{C} \otimes \Beta \mathrm{D}$ for any matrices $\Alpha, \Beta, \mathrm{C}, \mathrm{D}$ such that the products $\Alpha \mathrm{C}$ and $\Beta \mathrm{D}$ exist. 
    \item \textbf{Transpose.} We have that $\left(\Alpha \otimes \Beta\right)^\top = \Alpha^\top \otimes \Beta^\top$ for any matrices $\Alpha, \Beta$.
    \item \textbf{Vectorization.} Let $\Alpha \in \real^{n \times m}, \Beta \in \real^{m \times p}$, then  $\operatorname {Vec}(\Alpha \Beta)=(I_{p} \otimes \Alpha)\operatorname {Vec} (\Beta)=(\Beta^\top \otimes I_{n})\operatorname {Vec} (\Alpha)$.
    \item For any matrix $\Alpha$, we have that $I_1 \otimes \Alpha = \Alpha$.
    \item For any positive semi-definite matrices $\Alpha$ and $\Beta$, we have that $ \lambda_1(\Alpha \otimes \Beta) = \lambda_1(\Alpha) \lambda_1(\Beta)$.
    \item For any matrix $\Alpha$ and any positive semi-definite matrix $\Beta$ such that the product $\Alpha^\top \Beta \Alpha$ exists, the following inequality holds  $\lambda_1(\Alpha^\top \Beta \Alpha) \leq \lambda_1(\Beta ) \lambda_1(\Alpha^\top \Alpha)$.
\end{enumerate}

\section{MIXED-EFFECT MULTI-ARMED BANDIT} \label{additional_posteriors}
We introduce the mixed-effect multi-armed bandit setting. We then provide the effect posterior and action posteriors for this setting. The results below can be derived from \cref{thm:pt_gaussian,thm:pti_gaussian} by setting $d=1$ and $X_t=1$ for all $t \in [n]$.

\subsection{Mixed-Effect Multi-Armed Bandit}
\label{subsec:gaussian_linear_bandits}

For scalar effect parameters, $\psi_{*, \ell} \in \real$ holds for all $\ell \in [L]$, a natural effect prior $Q_{0}$ is a multivariate Gaussian with mean $\mu_{\Psi} \in \real^L$ and covariance $\Sigma_{\Psi} \in \real^{L \times L}$. The action prior $P_{0, i}$ is a univariate Gaussian with mean $\sum_{\ell =1}^L b_{i, \ell} \psi_{*, \ell} = b_i^\top \Psi_* \in \real$ and variance $\sigma_{0, i}^2>0$. This is a non-contextual setting ($X_t = 1$ for any $t \in [n]$) and thus the whole model reads
\begin{talign}\label{eq:gaussian_linear_bandits_model}
    \Psi_{*} & \sim \cN(\mu_{\Psi}, \Sigma_{\Psi})\,, \\ 
    \theta_{*, i} \mid   \Psi_{*} & \sim \cN\left( b_i^\top \Psi_*, \,  \sigma_{0, i}^2\right)\,, & \forall i \in [K]\,, \nonumber\\ 
    Y_t  \mid A_t, \theta_{*} & \sim \cN(\theta_{*, A_t} ,  \sigma^2)\,, & \forall t \in [n]\,. \nonumber
\end{talign}

\subsection{Posteriors for Mixed-Effect Multi-Armed Bandit}
\label{sec:mab posterior}

Fix round $t \in [n]$, and recall that $S_{t, i}$ are the rounds where action $i$ is taken up to round $t$. We introduce $N_{t, i} =  |S_{t, i}|$ as the number of times that action $i$ is taken up to round $t$ and $B_{t, i} = \sum_{\ell \in S_{t, i}} Y_\ell$ as the total reward of action $i$ up to round $t$. Note that $B_{t, i}$ and weight vectors $b_i$ are unrelated. We derive in \eqref{eq:setting1_posterior} and \eqref{eq:setting1_conditional_posterior} the effect posterior $Q_t$ and the action posteriors $P_{t, i}$ for the model in \eqref{eq:gaussian_linear_bandits_model}.

\begin{proposition} For any round $t \in [n]$, the joint effect posterior is a multivariate Gaussian $Q_t = \cN(\bar{\mu}_t, \bar{\Sigma}_t)$, where 
\begin{align}\label{eq:setting1_posterior}
    \bar{\Sigma}_t^{-1} &= \Sigma_{\Psi}^{-1} + \sum_{i\in[K]} \frac{N_{t, i}}{N_{t, i} \sigma_{0, i}^2 + \sigma^2} b_i b_i^\top\,, \qquad 
    \bar{\mu}_t = \bar{\Sigma}_t\left( \Sigma_{\Psi}^{-1} \mu_{\Psi} + \sum_{i\in[K]} \frac{B_{t, i}}{N_{t, i}\sigma_{0, i}^2+\sigma^2}  b_i\right)\,.
\end{align}
Moreover, for any action $i \in [K]$, and effect parameters $\Psi_t$, the action posterior is a univariate Gaussian $P_{t, i}(\cdot \mid \Psi_t) = \cN(\cdot; \tilde{\mu}_{t, i}, \tilde{\sigma}_{t, i}^2)$, where
\begin{align}\label{eq:setting1_conditional_posterior}
    \tilde{\sigma}_{t, i}^{-2} & = \frac{1}{\sigma_{0, i}^2}  + \frac{N_{t, i}}{\sigma^2} \,, \qquad
      \tilde{\mu}_{t, i} = \tilde{\sigma}_{t, i}^{2} \left( \frac{\Psi_t^\top b_i}{\sigma_{0, i}^{2}} +\frac{B_{t, i}}{\sigma^2} \right)\,.
\end{align}
\end{proposition} 

The effect posterior is additive in actions and can be interpreted as follows. Each action is a single noisy observation in its estimate. The \emph{maximum likelihood estimate (MLE)} of the expected reward of action $i$, $B_{t, i} / N_{t, i}$, contributes to \eqref{eq:setting1_posterior} proportionally to its precision, $N_{t, i} / (N_{t, i} \sigma_{0, i}^2 + \sigma^2)$. The contribution is weighted by $b_i$, which are the weights used to generate $\theta_{*, i}$. The action posterior in \eqref{eq:setting1_conditional_posterior} has a standard form. Note that its prior mean depends on effect parameters $\Psi_t$, which are sampled.

\section{POSTERIOR DERIVATIONS FOR MIXED-EFFECT LINEAR BANDIT}\label{app:posterior_derivations}

Here we provide the derivations of the effect posterior and action posteriors for the setting introduced in \cref{subsec:contextual_gaussian_bandits} and summarized in \eqref{eq:contextual_gaussian_model}. Precisely, we present the proof for \cref{thm:pt_gaussian} in \cref{app:effect posterior-derivation} and the proof of \cref{thm:pti_gaussian} in \cref{app:conditional-posterior-derivation}.

\subsection{Effect Posterior Derivation}\label{app:effect posterior-derivation}

\begin{proof}[Proof of \cref{thm:pt_gaussian} (derivation of $Q_t$)]
First, from basic properties of matrix operations in \cref{sub:preliminary_results}, we have that $ \sum_{\ell \in [L]} b_{i, \ell} \psi_{*, \ell} = \Gamma_i \Psi_*$ where $\Psi_* = (\psi_{*,\ell})_{\ell \in [L]} \in \real^{Ld}$ and $\Gamma_i = b_i^\top \otimes I_d$ (refer to \cref{subsec:regre_proof_1} for derivation detail). Thus our model can be written as
\begin{align}
    \Psi_* & \sim \cN(\mu_{\Psi}, \Sigma_{\Psi}) \,,\nonumber\\
    \theta_{*, i}  \mid \Psi_* &\sim \cN(\Gamma_i \Psi_*, \Sigma_{0, i})\,, & \forall i \in [K]\,, \nonumber\\
    Y_\ell  \mid X_\ell, \theta_{*, A_\ell} & \sim \cN(X_\ell^\top \theta_{*, A_\ell}, \sigma^2)\,, & \forall \ell \in [t]\,.
\end{align}
It follows that the joint effect posterior in round $t$ reads
\begin{align}
    Q_{t}(\Psi)  & \propto \condprob{H_t }{\Psi_{*} = \Psi} Q_0(\Psi)
  \stackrel{(i)}{=}  \prod_{i \in [K]} \condprob{H_{t, i} }{\Psi_{*} = \Psi}  Q_0(\Psi)\,,\\
  &= \prod_{i \in [K]} \int_{\theta_i} \condprob{H_{t, i},\theta_{*, i}= \theta_i}{\Psi_{*} = \Psi}\dif \theta_i Q_0(\Psi)
  = \prod_{i \in [K]} \int_{\theta_i} \LL_{t, i}(\theta_i) P_{0, i}\left(\theta_i \mid \Psi \right) \dif \theta_i Q_0(\Psi)\,,\\
    &= \prod_{i \in [K]} \int_{\theta_i} \LL_{t, i}(\theta_i) \cN \left(\theta_i ; \Gamma_i \Psi, \Sigma_{0, i}\right) \dif \theta_i \cN(\Psi;\mu_{\Psi}, \Sigma_{\Psi})\,, \nonumber\\
   & = \prod_{i \in [K]} \int_{\theta_i}  \left(\prod_{\ell \in S_{t, i}} \cN(Y_\ell; X_\ell^\top \theta_i, \sigma^2)\right) \cN \left(\theta_i ; \Gamma_i \Psi, \Sigma_{0, i}\right) \dif \theta_i \cN(\Psi;\mu_{\Psi}, \Sigma_{\Psi})\,.
\end{align}
Here $(i)$ follows from the fact that $\theta_{*, i}$ for $i \in [K]$ are conditionally independent given $\Psi_*=\Psi$ and that $H_{t, i}$ depends on $\Psi_*$ only through $\theta_{*, i}$. Now we compute the quantity $\int_{\theta_i} \left(\prod_{\ell \in S_{t, i}} \cN(Y_\ell; X_\ell^\top \theta_i, \sigma^2)\right) \cN \left(\theta_i ; \Gamma_i \Psi, \Sigma_{0, i}\right) \dif \theta_i$ using \cref{lemma:gaussian_posterior_per_arm_contextual}. Precisely, we obtain that it is proportional to $\cN(\Psi; \bar{\mu}_{t, i}, \bar{\Sigma}_{t, i})$ where
\begin{align*}
    \bar{\Sigma}_{t, i}^{-1} &  = \Gamma_i^{\top} \left( \Sigma_{0, i} + G_{t, i}^{-1} \right)^{-1}\Gamma_i  , \\
    \bar{\mu}_{t, i} & =   \bar{\Sigma}_{t, i}\left( \Gamma_i^\top (\Sigma_{0, i} + G_{t, i}^{-1})^{-1} G_{t, i}^{-1} B_{t, i}\right) ,
\end{align*}
and 
\begin{align*}
    &G_{t,i} = \sigma^{-2} \sum_{\ell \in S_{t, i}} X_\ell X_\ell^\top\,, 
    &B_{t, i} = \sigma^{-2} \sum_{\ell \in S_{t, i}} Y_\ell X_\ell\,.
\end{align*}
This means that the effect posterior $Q_t$ is the product of $K+1$ multivariate Gaussian distributions $\cN(\mu_{\Psi}, \Sigma_{\Psi}), \cN(\bar{\mu}_{t, 1}, \bar{\Sigma}_{t, 1}), \ldots, \cN(\bar{\mu}_{t, K}, \bar{\Sigma}_{t, K})$. Thus, the effect posterior $Q_t$ is also a multivariate Gaussian distribution $\cN(\bar{\mu}_t, \bar{\Sigma}_t^{-1})$, where
\begin{align*}
    \bar{\Sigma}_t^{-1} &=  \Sigma_{\Psi}^{-1} + \sum_{i=1}^K\bar{\Sigma}_{t, i}^{-1} = \Sigma_{\Psi}^{-1} + \sum_{i=1}^K \Gamma_i^{\top} \left( \Sigma_{0, i} + G_{t, i}^{-1} \right)^{-1}\Gamma_i\,, \\
   \bar{\mu}_t &= \bar{\Sigma}_t\left( \Sigma_{\Psi}^{-1} \mu_{\Psi} +\sum_{i=1}^K\bar{\Sigma}_{t, i}^{-1}\bar{\mu}_{t, i}\right) = \bar{\Sigma}_t\left( \Sigma_{\Psi}^{-1} \mu_{\Psi} +\sum_{i=1}^K   \Gamma_i^\top (\Sigma_{0, i} + G_{t, i}^{-1})^{-1} G_{t, i}^{-1} B_{t, i}\right)\,.
\end{align*}
Moreover, we use the properties in \cref{sub:preliminary_results} and that $\Gamma_i = b_i^\top \otimes I_d$ to rewrite the terms as
\begin{align*}
    \Gamma_i^{\top} \left( \Sigma_{0, i} + G_{t, i}^{-1} \right)^{-1}\Gamma_i & = b_i b_i^\top \otimes \left( \Sigma_{0, i} + G_{t, i}^{-1} \right)^{-1}\,,\\
    \Gamma_i^\top (\Sigma_{0, i} + G_{t, i}^{-1})^{-1} G_{t, i}^{-1} B_{t, i} &= b_i \otimes \left( (\Sigma_{0, i} + G_{t, i}^{-1})^{-1} G_{t, i}^{-1} B_{t, i}\right)\,.
\end{align*}
This concludes the proof.
\end{proof}

To reduce clutter, we fix an action $i \in [K]$ and a round $t \in [n]$ and drop subindexing by $i$ and $t$ in the following lemma. 
In summary, there exist $i \in [K]$ and $t \in [n]$ such that we have the following correspondences:
\begin{align*}
    \Gamma \gets \Gamma_i\,, \quad \ \Sigma_{0} \gets \Sigma_{0, i}\,, \quad N \gets N_{t, i}\,, \quad \theta \gets \theta_i\,, \quad (X_\ell, Y_\ell)_{\ell \in [N]} \gets (X_\ell, Y_\ell)_{\ell \in S_{t,i}}\,,
\end{align*}

\begin{lemma}[Gaussian posterior update]\label{lemma:gaussian_posterior_per_arm_contextual} Let $\Gamma \in \real^{d \times Ld}\,,$ $\Sigma_0 \in \real^{d \times d}\,,$ and $\sigma > 0$ then we have that 
\begin{align*}
\int_\theta \left(\prod_{\ell = 1}^N \cN(Y_\ell; X_\ell^\top \theta, \sigma^2)\right)
  \cN(\theta; \Gamma \Psi, \Sigma_0) \dif\theta \propto \cN\left(\Psi; \mu_N, \Sigma_N \right)\,. \end{align*}
where  
 \begin{align*}
    \Sigma_N^{-1} & = \Gamma^{\top} \left( \Sigma_{0} + G_{N}^{-1} \right)^{-1}\Gamma\,, \\
     \mu_N & = \Sigma_N\left(  \Gamma^\top (\Sigma_{0} + G_{N}^{-1})^{-1} G_{N}^{-1} B_{N}\right)\,,
\end{align*}
and
\begin{align*}
    G_N &= \sigma^{-2} \sum_{k=1}^N X_k X_k^\top\,, \quad B_N = \sigma^{-2} \sum_{k = 1}^N Y_k X_k\,.
\end{align*}
\end{lemma}

\begin{proof}
Let $v
  = \sigma^{-2}\,, \quad
  \Lambda_0
  = \Sigma_0^{-1}\,.$ We denote the integral in the lemma by $f(\Psi)$. It follows that
\begin{align*}
 f(\Psi) & =  \int_\theta \left(\prod_{\ell = 1}^N \cN(Y_\ell; X_\ell^\top \theta, \sigma^2)\right)
  \cN(\theta; \Gamma \Psi, \Sigma_0) \dif\theta \,, \\
  & \propto \int_\theta \exp\left[
  - \frac{1}{2} v \sum_{\ell = 1}^N (Y_\ell - X_\ell^\top \theta)^2 -
  \frac{1}{2} (\theta - \Gamma \Psi)^\top \Lambda_0 (\theta - \Gamma \Psi)\right] \dif \theta \,,\\
  & = \int_\theta \exp\left[- \frac{1}{2}
  \left(v \sum_{\ell = 1}^N (Y_\ell^2 - 2 Y_\ell \theta^\top X_\ell + (\theta^\top X_\ell)^2) +
  \theta^\top \Lambda_0 \theta - 2 \theta^\top \Lambda_0 \Gamma \Psi + (\Gamma \Psi)^\top \Lambda_0 (\Gamma \Psi) \right)\right] \dif \theta \,,\\
  & \propto \int_\theta \exp\left[- \frac{1}{2}
  \left(\theta^\top \left( v \sum_{\ell=1}^N X_\ell X_\ell^\top + \Lambda_0 \right) \theta - 2 \theta^\top \left(v \sum_{\ell = 1}^N Y_\ell X_\ell + \Lambda_0 \Gamma \Psi\right) + (\Gamma \Psi)^\top \Lambda_0 (\Gamma \Psi)\right) \right] \dif \theta \,.
 \end{align*}
To reduce clutter, let 
\begin{align*}
    G_N = v \sum_{\ell=1}^N X_\ell X_\ell^\top\,, \quad V_N = \left(G_N + \Lambda_0\right)^{-1}\,, \quad U_N = V_N^{-1}\,,  \\ B_N = v \sum_{\ell = 1}^N Y_\ell X_\ell \quad \text{ and } \quad \beta_N = V_N \left(B_N  + \Lambda_0 \Gamma \Psi\right)\,.
\end{align*}
We have that $U_N V_N = V_N U_N = I_{d}\,,$ and thus
 \begin{align*}
  f(\Psi) & \propto \int_\theta \exp\left[- \frac{1}{2}
  \left(\theta^\top U_N \theta - 2 \theta^\top U_N V_N\left(B_N + \Lambda_0 \Gamma \Psi\right) + (\Gamma \Psi)^\top \Lambda_0 (\Gamma \Psi)\right) \right] \dif \theta \,,\\
  & = \int_\theta \exp\left[- \frac{1}{2}
  \left(\theta^\top U_N \theta - 2 \theta^\top U_N \beta_N +  (\Gamma \Psi)^\top \Lambda_0 (\Gamma \Psi)\right) \right] \dif \theta \,,\\
  & = \int_\theta \exp\left[- \frac{1}{2}
  \left( (\theta - \beta_N)^\top U_N (\theta - \beta_N) - \beta_N^\top U_N \beta_N + (\Gamma \Psi)^\top \Lambda_0 (\Gamma \Psi)\right) \right] \dif \theta \,, \\
  & \propto  \exp\left[- \frac{1}{2}
  \left( - \beta_N^\top U_N \beta_N + (\Gamma \Psi)^\top \Lambda_0 (\Gamma \Psi)\right) \right]\,,\\
  & =  \exp\left[- \frac{1}{2}
  \left( -  \left(B_N + \Lambda_0 \Gamma \Psi\right)^\top  V_N \left(B_N + \Lambda_0 \Gamma \Psi\right) + (\Gamma \Psi)^\top \Lambda_0 (\Gamma \Psi)\right) \right] \,,\\
  & \propto  \exp\left[- \frac{1}{2}
  \left( \Psi^{\top} \Gamma^{\top} \left( \Lambda_0 - \Lambda_0 V_N \Lambda_0 \right) \Gamma \Psi - 2 \Psi^\top \left(\Gamma^\top \Lambda_0 V_N B_N\right) \right) \right] \,,\\
  & = \exp\left[- \frac{1}{2} \Psi^\top \Sigma_N^{-1} \Psi + \Psi^\top \Sigma_N^{-1} \mu_N \right]\,,
\end{align*}
where 
\begin{align}\label{eq:effect_posterior_formulas}
    \Sigma_N^{-1} & = \Gamma^{\top} \left( \Lambda_0 - \Lambda_0 V_N \Lambda_0 \right)\Gamma = \Gamma^{\top} \left( \Lambda_0^{-1} + G_N^{-1} \right)^{-1}\Gamma\,, \nonumber \\
    \Sigma_N^{-1} \mu_N & = \left( \Gamma^\top \Lambda_0 V_N B_N \right) =   \Gamma^\top (\Sigma_{0} + G_{N}^{-1})^{-1} G_{N}^{-1} B_{N}\,.
\end{align} 
We use the Woodbury matrix identity to get the second equalities which concludes the proof.
\end{proof}

\subsection{Action Posterior Derivation for Mixed-Effect Linear Bandit}\label{app:conditional-posterior-derivation}

\begin{proof}[Proof of \cref{thm:pti_gaussian} (Derivation of $P_{t,i}$)]
This proposition is a direct application \cref{lemma:posterior_theta}; in which case we get that the posterior $P_{t, i}$ is a multivariate Gaussian distribution $\cN(\tilde{\mu}_{t, i}, \tilde{\Sigma}_{t, i})$, where
\begin{align*}
    \tilde{\Sigma}_{t, i}^{-1} &=  G_{t, i} + \Sigma_{0, i}^{-1}\,, \nonumber\\
    \tilde{\mu}_{t, i} &= \tilde{\Sigma}_{t, i} \left( B_{t, i} +  \Sigma_{0, i}^{-1} \sum_{\ell=1}^L b_{i, \ell} \psi_{t, \ell}\right)\,.
\end{align*}
\end{proof}

To reduce clutter, we consider a fixed action $i \in [K]$ and round $t \in [n]$, and drop subindexing by $t$ and $i$ in \cref{lemma:posterior_theta}. In summary, there exist $i \in [K]$ and $t\in [n]$ such that we have the following correspondences:
\begin{align*}
    b_\ell \gets b_{i, \ell}\,, \quad \ \Sigma_{0} \gets \Sigma_{0, i}\,, \quad N \gets N_{t, i}\,, \quad \theta_* \gets \theta_{*, i}\,, \quad (X_\ell, Y_\ell)_{\ell \in [N]} \gets (X_\ell, Y_\ell)_{\ell \in S_{t,i}}\,.
\end{align*}

\begin{lemma}\label{lemma:posterior_theta}
Consider the following model
\begin{align*}
    \theta_*  \mid \Psi_* & \sim \cN\left(\sum_{\ell=1}^L b_\ell \psi_{*, \ell}, \Sigma_0\right)\,,\\
    Y_\ell  \mid X_\ell, \theta_* & \sim \cN\left(X_\ell^\top \theta_*, \sigma^2\right)\,, & \forall \ell \in [N]\,.
\end{align*}
Let $H = \set{X_1, Y_1, \ldots, X_N, Y_N }$ then we have that $\condprob{\theta_{*} = \theta}{\Psi_*=\Psi, H} = \cN\left(\theta; \tilde{\mu}_{N}, \tilde{\Sigma}_{N}\right),$
where 
\begin{align*}
\tilde{\Sigma}^{-1}_{N} &= \sigma^{-2}\sum_{\ell =1}^N X_\ell X_\ell^\top + \Sigma_{0}^{-1}\,,\\
   \tilde{\mu}_{N} &=   \tilde{\Sigma}_{N}\left( \sigma^{-2} \sum_{\ell =1}^N X_\ell Y_\ell + \Sigma_{0}^{-1} \sum_{\ell=1}^L b_\ell \psi_\ell \right)\,.
\end{align*}

\end{lemma}

\begin{proof}
Let $v
  = \sigma^{-2}\,, \quad
  \Lambda_0
  = \Sigma_0^{-1}\,.$ Then the action posterior decomposes as
\begin{align*}
    &\condprob{\theta_{*} = \theta}{\Psi_*=\Psi, H}\,, \\
    &\propto \condprob{H}{\Psi_*=\Psi, \theta_{*} = \theta} \condprob{\theta_{*} = \theta}{\Psi_*=\Psi}\,,  \\
    &= \condprob{H}{\theta_{*} = \theta} \condprob{\theta_{*} = \theta}{\Psi_*=\Psi}\,, \quad (\text{$H$ depends on $\Psi_*$ only through $\theta_*$})\,,  \\
    & = \prod_{\ell =1}^N \cN(Y_\ell; X_\ell^\top \theta, \sigma^2) \cN(\theta; \sum_{\ell=1}^L b_\ell \psi_\ell , \Sigma_{0})\,,\\
    & = \exp\left[- \frac{1}{2}
  \left(v \sum_{\ell =1}^N (Y_\ell^2 - 2 Y_\ell X_\ell^\top \theta + (X_\ell^\top\theta)^2) + \theta^\top \Lambda_{0} \theta - 2 \theta^\top \Lambda_{0} \sum_{\ell=1}^L b_\ell \psi_\ell + \left(\sum_{\ell=1}^L b_\ell \psi_\ell\right)^\top \Lambda_{0} \left(\sum_{\ell=1}^L b_\ell \psi_\ell\right)\right)\right] \,,\\
  &\propto \exp\left[- \frac{1}{2}
  \left( \theta^\top(v\sum_{\ell =1}^N X_\ell X_\ell^\top + \Lambda_{0}) \theta - 2 \theta^\top \left(v \sum_{\ell =1}^N X_\ell Y_\ell + \Lambda_{0} \sum_{\ell=1}^L b_\ell \psi_\ell \right)\right)\right]\,,\\
  &\propto \cN\left(\theta; \tilde{\mu}_{N}, \left(\tilde{\Lambda}_{N}\right)^{-1}\right)\,,
\end{align*}
where $\tilde{\Lambda}_{N} = v\sum_{\ell =1}^N X_\ell X_\ell^\top + \Lambda_{0}\,,$ and $ \tilde{\Lambda}_{N} \tilde{\mu}_{N} = v \sum_{\ell =1}^N X_\ell Y_\ell + \Lambda_{0} \sum_{\ell=1}^L b_\ell \psi_\ell\,.$
\end{proof}

\section{REGRET PROOFS}\label{proof:regret_proof}
In this section, we prove a more general version of \cref{thm:regret}. First, we provide a compact formulation of our problem in \cref{subsec:regre_proof_1}. Next, we use the total covariance decomposition to derive the covariance of $\condprob{\theta_{*,i} = \theta}{H_t}$ in \cref{subsec:regre_proof_2}. Finally, we provide preliminary eigenvalue results in \cref{seubsec:prem_ineq} to proceed with the regret proof in \cref{subsec:regret_proof_3}.

\subsection{Problem Reformulation for Regret Analysis}\label{subsec:regre_proof_1}

Here, we aim at rewriting our model in a compact form to simplify regret analysis. We first introduce $K$ i.i.d. multivariate Gaussian variables $Z_i \sim \cN(0, \Sigma_{0, i})$ for $i \in [K]$, and the following matrix
\begin{align*}
    \Psi_{\text{mat}, *} &= [\psi_{*, 1}, \ldots, \psi_{*, L}] \in \real^{d \times L}\,.
\end{align*}
First, we have that $\operatorname{Vec}(\Psi_{\text{mat}, *})=\Psi_*$ where $\Psi_*$ is defined in \cref{sec:setting}. Moreover notice that $\sum_{\ell=1}^L b_{i, \ell} \psi_{*, \ell} = \Psi_{\text{mat}, *} b_i$ and thus given matrix $\Psi_{\text{mat}, *}$ we have that
\begin{align}\label{eq:params_per_arm}
\theta_{*, i} &= \Psi_{\text{mat}, *} b_i + Z_i\,, &\forall i \in [K] \,.
\end{align}
We vectorize \eqref{eq:params_per_arm} to obtain
\begin{align}
    \theta_{*,i} = \operatorname {Vec}(\theta_{*,i}) = \operatorname {Vec}(\Psi_{\text{mat}, *} b_i + Z_i) =\operatorname {Vec}(\Psi_{\text{mat}, *} b_i) + Z_i\,,
\end{align}
 where we used that if $X \in \real^d$ (a column vector), then $X = \operatorname {Vec}(X)$ and that $\operatorname {Vec}(\cdot)$ is a linear transformation. Also, we know from (c) in \cref{sub:preliminary_results} that $\operatorname {Vec}(\Alpha \Beta)=(\Beta^\top \otimes I_{n})\operatorname {Vec} (\Alpha)$ for any $\Alpha \in \real^{n \times m}, \Beta \in \real^{m \times p}$. Therefore,
\begin{align}\label{eq:vec_params}
    \theta_{*,i} = \Gamma_i \Psi_* + Z_i\,,
\end{align}
where $\Gamma_i = b_i^\top \otimes I_d$ and we used that $\operatorname {Vec}(\Psi_{\text{mat}, *}) = \Psi_*$. It follows that 
\begin{align}\label{eq:reformulation_per_arm}
     \theta_{*,i} \mid \Psi_* \sim \cN(\Gamma_i  \Psi_*, \Sigma_{0,i})\,,
\end{align}

This allows us to rewrite our model as a single-parent hierarchical model
\begin{align}\label{eq:model_rewritten}
    \Psi_{*} & \sim \cN(\mu_{\Psi}, \Sigma_{\Psi})\,, \\ 
         \theta_{*,i} \mid \Psi_* & \sim \cN(\Gamma_i  \Psi_*, \Sigma_{0,i})\,, & \forall i \in [K]\,,\nonumber\\ 
    Y_t  \mid X_t, \theta_{*, A_t} & \sim \cN(X_t^\top \theta_{*, A_t} ,  \sigma^2)\,, & \forall t \in [n]\,. \nonumber
\end{align}

\subsection{Derivation of $\condprob{\theta_{*,i} = \theta}{H_t}$}\label{subsec:regre_proof_2}
Let 
\begin{align*}
  G_{t, i}
  &= \sigma^{-2} \sum_{\ell \in S_{t, i}} X_\ell X_\ell^\top \,,
  \quad
  B_{t, i}
  = \sigma^{-2} \sum_{\ell \in S_{t, i}} Y_\ell X_\ell \,.
\end{align*}

\begin{lemma}[Covariance of $ \condprob{\theta_{*,i}= \theta}{H_t}$]\label{lemma:gaussian_covariance} Consider the model in \eqref{eq:contextual_gaussian_model} and let $\Psi_* \mid H_t \sim \cN(\bar{\mu}_t, \bar{\Sigma}_t)$, then we have
\begin{align*}
  \hat{\Sigma}_{t,i} &= \condcov{\theta_{*, i}}{H_t} = \tilde{\Sigma}_{t,i}  +\tilde{\Sigma}_{t,i} \Sigma_{0,i}^{-1}  \Gamma_i \bar{\Sigma}_{t} \Gamma_i^\top \Sigma_{0,i}^{-1} \tilde{\Sigma}_{t,i}\,, & \forall i \in [K]\,.
\end{align*}
where $\tilde{\Sigma}_{t,i} =  \condcov{\theta_{*,i}}{\Psi_*, H_t}=\left( G_{t,i} + \Sigma_{0,i}^{-1} \right)^{-1}$.
\end{lemma}
\begin{proof}
Let $\Lambda_{0,i} =  \Sigma_{0,i}^{-1}$. \cref{thm:pti_gaussian} and the fact that $ \sum_{\ell \in [L]} b_{i, \ell} \psi_{*, \ell} = \Gamma_i \Psi_*$ where $\Gamma_i = b_i^\top \otimes I_d$ (\cref{subsec:regre_proof_1}) yield
\begin{align*}
    \condcov{\theta_{*,i}}{\Psi_*, H_t} &= \left( G_{t,i} + \Lambda_{0,i} \right)^{-1} \\
    \condE{\theta_{*,i}}{\Psi_*, H_t} & = \condcov{\theta_{*,i}}{\Psi_*, H_t} \left( B_{t,i}  + \Lambda_{0,i} \Gamma_i \Psi_* \right)
\end{align*}
First, given $H_t$, $\condcov{\theta_{*,i}}{\Psi_*, H_t} =\left( G_{t,i} + \Lambda_{0,i} \right)^{-1}$ is constant (does not depend on $\Psi_*$). Thus 
\begin{align*}   \condE{\condcov{\theta_{*,i}}{\Psi_*, H_t}}{H_t} = \condcov{\theta_{*,i}}{\Psi_*, H_t} = \left( G_{t,i} + \Lambda_{0,i} \right)^{-1}\,.
\end{align*}
In addition, given $H_t$, both $\left( G_{t,i} + \Lambda_{0,i} \right)^{-1}$ and $B_{t,i}$ are constant. Thus
\begin{align*}
    \condcov{\condE{\theta_{*,i}}{\Psi_*, H_t}}{H_t} &= \condcov{\condcov{\theta_{*,i}}{\Psi_*, H_t}\Lambda_{0,i} \Gamma_i \Psi_*}{H_t}\\
    &= \left( G_{t,i} + \Lambda_{0,i} \right)^{-1} \Lambda_{0,i} \Gamma_i \condcov{\Psi_*}{H_t} \Gamma_i^\top \Lambda_{0,i} \left( G_{t,i} + \Lambda_{0,i} \right)^{-1}\\
    &= \left( G_{t,i} + \Lambda_{0,i} \right)^{-1} \Lambda_{0,i} \Gamma_i \bar{\Sigma}_t \Gamma_i^\top \Lambda_{0,i} \left( G_{t,i} + \Lambda_{0,i} \right)^{-1}.
\end{align*}
Finally, total covariance decomposition \citep{weiss05probability} concludes the proof.
\end{proof}

\subsection{Preliminary Eigenvalues Results}\label{seubsec:prem_ineq}
Next we present some preliminary upper bounds on the maximum eigenvalues of our covariance matrices.
\setlist{nolistsep}
\begin{itemize}[noitemsep]
    \item \textbf{Definitions:} Let $\lambda_{1, 0} = \max_{i \in [K]}  \lambda_1(\Sigma_{0, i})\,,$ $\lambda_{d, 0}  = \min_{i \in [K]}  \lambda_d(\Sigma_{0, i})\,,$ $\lambda_{1, \Psi} = \lambda_1(\Sigma_{\Psi})\,,$ and $\kappa_b = \max_{i \in [K]} \normw{b_i}{2}^2 $.
    \item \textbf{upper bound of $\lambda_1(\Gamma_i \Gamma_i^\top)$:}\begin{align}\label{eq:max_eigenvalue_mix_weights}
   \lambda_1(\Gamma_i \Gamma_i^\top) & \leq \kappa_b \,, & \forall i \in [K]\,.
    \end{align} 
    Similarly, we have that
    \begin{align}
   \lambda_1(\Gamma_i^\top \Gamma_i) & \leq \kappa_b \,, & \forall i \in [K]\,.
    \end{align} 
    \item \textbf{upper bound of $\lambda_1(\hat{\Sigma}_{t, i})$:} \begin{align}\label{eq:max_eigenvalue_covariance}
    \lambda_1(\hat{\Sigma}_{t, i}) & \leq \lambda_{1, 0} + \frac{\lambda_{1, 0}^2 \lambda_{1, \Psi} \kappa_b }{\lambda_{d, 0}^2} \,, & \forall i \in [K]\,.
\end{align} 
    \item \textbf{upper bound of $\lambda_1(\Sigma_\Psi^\frac{1}{2} \bar{\Sigma}_{n+1}^{-1} \Sigma_\Psi^\frac{1}{2})$:}\begin{align}\label{eq:max_eigenvalue_effect_covariance}
    \lambda_1(\Sigma_\Psi^\frac{1}{2} \bar{\Sigma}_{n+1}^{-1} \Sigma_\Psi^\frac{1}{2}) & \leq 1 + K \lambda_{1, \Psi} \kappa_b \Big( \frac{1}{\lambda_{d, 0}} - \frac{1}{\lambda_{1, 0}^2\big(\frac{\kappa_x n}{\sigma^2} + \frac{1}{\lambda_{d, 0}} \big)} \Big) \,.
\end{align} 
\end{itemize}

\begin{proof}
We start with \eqref{eq:max_eigenvalue_mix_weights}. First, recall that $\Gamma_i = b_i^\top \otimes I_d$ for any $i \in [K]$. Thus $\Gamma_i \Gamma_i^\top = (b_i^\top \otimes I_d) (b_i \otimes I_d) = \normw{b_i}{2}^2 I_d$ for any $i \in [K]$. Then $\lambda_1(\Gamma_i \Gamma_i^\top) = \normw{b_i}{2}^2 \leq \kappa_b\,.$ The second result follows from the fact that $\lambda_1(\Gamma_i \Gamma_i^\top)=\lambda_1(\Gamma_i^\top \Gamma_i)$. 

Now we prove the result in \eqref{eq:max_eigenvalue_covariance}. This follows from the expression of $\hat{\Sigma}_{t,i}$ in \cref{lemma:gaussian_covariance}. Precisely, we have that 
\begin{align*}
  \hat{\Sigma}_{t,i} & = \tilde{\Sigma}_{t,i}  +\tilde{\Sigma}_{t,i} \Sigma_{0,i}^{-1}  \Gamma_i \bar{\Sigma}_{t} \Gamma_i^\top \Sigma_{0,i}^{-1} \tilde{\Sigma}_{t,i}\,, & \forall i \in [K]\,.
\end{align*}
where $\tilde{\Sigma}_{t,i} =\left( G_{t,i} + \Sigma_{0,i}^{-1} \right)^{-1}$. Thus Weyl's inequality combined with the properties in \cref{sub:preliminary_results} yields that
\begin{align*}
    \lambda_1(\hat{\Sigma}_{t, i}) & \leq \lambda_1(\tilde{\Sigma}_{t,i}) + \lambda_1(\tilde{\Sigma}_{t,i}) \lambda_1({\Sigma}_{0,i}^{-1})\lambda_1(\Gamma_i\bar{\Sigma}_{t}\Gamma_i^\top) \lambda_1({\Sigma}_{0,i}^{-1}) \lambda_1(\tilde{\Sigma}_{t,i}) \leq  \lambda_{1, 0} + \frac{\lambda_{1, 0}^2 \lambda_{1, \Psi} \kappa_b }{\lambda_{d, 0}^2}
\end{align*} 
In the last inequality, we used that $\lambda_1(\Gamma_i\bar{\Sigma}_{t}\Gamma_i^\top) \leq \lambda_1(\bar{\Sigma}_{t}) \lambda(\Gamma_i \Gamma_i^\top),$ ((f) in \cref{sub:preliminary_results}),  $\lambda_1({\Sigma}_{0,i}^{-1}) \leq \frac{1}{\lambda_{d, 0}}\,,$ and $\lambda_1(\tilde{\Sigma}_{t,i}) \leq \lambda_{1,0}$.

Finally, we prove the result in \eqref{eq:max_eigenvalue_effect_covariance}. First, we rewrite the precision matrix of the effect posterior $\bar{\Sigma}_t^{-1}$ using the compact notation introduced in \cref{subsec:regre_proof_1}. Precisely, it follows from \eqref{eq:effect_posterior_formulas} that
\begin{align*}
    \bar{\Sigma}_t^{-1} \stackrel{(i)}{=} \Sigma_\Psi^{-1} + \sum_{i=1}^K \Gamma_{i}^{\top} \left( \Sigma_{0,i} + G_{t,i}^{-1} \right)^{-1}\Gamma_{i} &\stackrel{(ii)}{=}  \Sigma_\Psi^{-1} + \sum_{i=1}^K\Gamma_i^{\top} \left( \Sigma_{0, i}^{-1} - \Sigma_{0, i}^{-1} (G_{t, i} + \Sigma_{0, i}^{-1})^{-1} \Sigma_{0, i}^{-1} \right)\Gamma_i \,,\nonumber\\
    &\stackrel{(iii)}{=}  \Sigma_\Psi^{-1} + \sum_{i=1}^K\Gamma_i^{\top} \left( \Sigma_{0, i}^{-1} - \Sigma_{0, i}^{-1} \tilde{\Sigma}_{t, i} \Sigma_{0, i}^{-1} \right)\Gamma_i \,.
\end{align*}
The equality $(i)$ requires $G_{t,i}$ to be invertible and was only given in the main manuscript to ease the exposition. Note that $(i)$ is obtained by applying the Woodbury matrix identity on $(ii)$. In our proof, we use $(ii)$ and $(iii)$ which are the same; $(iii)$ follows from plugging $\tilde{\Sigma}_{t, i} = (G_{t, i} + \Sigma_{0, i}^{-1})^{-1}$ in $(ii)$. Then we have that
    \begin{align*}
    &\lambda_1(\Sigma_\Psi^\frac{1}{2} \bar{\Sigma}_{n+1}^{-1} \Sigma_\Psi^\frac{1}{2})\\
    &=\lambda_1\Big(I_{Ld} + \sum_{i=1}^K \Sigma_\Psi^\frac{1}{2}\Gamma_i^{\top} \left( \Sigma_{0, i}^{-1} - \Sigma_{0, i}^{-1} \tilde{\Sigma}_{n+1, i} \Sigma_{0, i}^{-1} \right)\Gamma_i \Sigma_\Psi^\frac{1}{2} \Big) \leq 1+ \lambda_{1, \Psi}\sum_{i=1}^K\lambda_1\Big(\Gamma_i^{\top} \left( \Sigma_{0, i}^{-1} - \Sigma_{0, i}^{-1} \tilde{\Sigma}_{n+1, i} \Sigma_{0, i}^{-1} \right)\Gamma_i\Big) \nonumber\,,\\
    &\leq 1+ \lambda_{1, \Psi} \sum_{i=1}^K\lambda_1(\Gamma_i^{\top}\Gamma_i) \lambda_1\left( \Sigma_{0, i}^{-1} - \Sigma_{0, i}^{-1} \tilde{\Sigma}_{n+1, i} \Sigma_{0, i}^{-1} \right) \leq 1+ \lambda_{1, \Psi}\sum_{i=1}^K\kappa_b \left( \lambda_1\left( \Sigma_{0, i}^{-1}\right) + \lambda_1\left(-\Sigma_{0, i}^{-1} \tilde{\Sigma}_{n+1, i} \Sigma_{0, i}^{-1} \right) \right)\nonumber\,,\\
    &\leq 1+ \lambda_{1, \Psi} \sum_{i=1}^K\kappa_b \left( \frac{1}{\lambda_{d, 0}} - \lambda_d\left(\Sigma_{0, i}^{-1} \tilde{\Sigma}_{n+1, i} \Sigma_{0, i}^{-1} \right) \right)\nonumber \leq  1+ \lambda_{1, \Psi} \sum_{i=1}^K\kappa_b \left( \frac{1}{\lambda_{d, 0}} - \lambda_d\left(\Sigma_{0, i}^{-1}\right)\lambda_d\left( \tilde{\Sigma}_{n+1, i} \right)\lambda_d\left(\Sigma_{0, i}^{-1} \right) \right)\nonumber\,,\\
    &\leq  1+ \lambda_{1, \Psi} \sum_{i=1}^K\kappa_b \left( \frac{1}{\lambda_{d, 0}} - \frac{1}{\lambda_{1, 0}^2}\lambda_d\left( \tilde{\Sigma}_{n+1, i} \right) \right)\nonumber = 1+ \lambda_{1, \Psi} \sum_{i=1}^K\kappa_b \left( \frac{1}{\lambda_{d, 0}} - \frac{1}{\lambda_{1, 0}^2\lambda_1\left( G_{n+1, i} + \Sigma_{0, i}^{-1} \right)} \right)\nonumber\,,\\
    &\leq  1+ \lambda_{1, \Psi} \sum_{i=1}^K\kappa_b \left( \frac{1}{\lambda_{d, 0}} - \frac{1}{\lambda_{1, 0}^2\left(\frac{\kappa_x n}{\sigma^2} + \frac{1}{\lambda_{d, 0}} \right)} \right) =  1+K \lambda_{1, \Psi} \kappa_b \left( \frac{1}{\lambda_{d, 0}} - \frac{1}{\lambda_{1, 0}^2\left(\frac{\kappa_x n}{\sigma^2} + \frac{1}{\lambda_{d, 0}} \right)} \right)\nonumber\,.
\end{align*}
\end{proof}

\subsection{Regret Proof}\label{subsec:regret_proof_3}
Here we prove a more general version of \cref{thm:regret} where we do not assume that the covariance matrices $\Sigma_{0, i}$ and $\Sigma_{\Psi}$ are diagonal. We still assume that there exists $\kappa_x>0$ such that $\normw{X_t}{2}^2 \leq \kappa_x$ for any $t \in [n]$.

\begin{theorem}[General version of \cref{thm:regret}] For any $\delta \in (0, 1)$, the Bayes regret of \emph{\alg} in the mixed-effect model in \cref{subsec:contextual_gaussian_bandits} is bounded as
\begin{align}
  \mathcal{B}\mathcal{R}(n)
  \leq \sqrt{2 n 
  \left( \mathcal{R}^{\textsc{a}}(n) + \mathcal{R}^{\textsc{e}}(n) \right) \log(1 / \delta)} +
  c n\delta\,,
\end{align}
with $c = \sqrt{ \frac{2}{\pi} \kappa_x\big( \lambda_{1, 0} + \frac{\lambda_{1, 0}^2 \lambda_{1, \Psi}\kappa_b }{\lambda_{d, 0}^2}\big)}K \,, \, \, \kappa_b = \max_{i \in [K]}\normw{b_i}{2}^2\,, \, \, \lambda_{1, 0}=\max_{i \in [K]}\lambda_1(\Sigma_{0, i})\,, \, \, \lambda_{d, 0} = \min_{i \in [K]}\lambda_d(\Sigma_{0, i})\,,$ $\lambda_{1, \Psi} = \lambda_1(\Sigma_{\Psi})$ and
\begin{talign*}
& \mathcal{R}^{\textsc{a}}(n) = dK c_\textsc{a} \log\big(1 + \frac{n\kappa_x \lambda_{1, 0}  }{\sigma^2 d}\big)\,, \, c_\textsc{a} = \frac{\kappa_x \lambda_{1, 0}}{\log(1 + \frac{\kappa_x \lambda_{1, 0}}{\sigma^{2}})}\,,\\ 
&\mathcal{R}^{\textsc{e}}(n) = dL c_\textsc{e} \log\Big(1+K \kappa_b \lambda_{1, \Psi} \Big( \frac{1}{\lambda_{d, 0}} - \frac{1}{\lambda_{1, 0}^2\big(\frac{\kappa_x n}{\sigma^2} + \frac{1}{\lambda_{d, 0}} \big)} \Big)\Big)\,,  \, c_\textsc{e}= \frac{\kappa_x \kappa_b \lambda_{1, 0}^2 \lambda_{1, \Psi} \big(1 + \frac{\kappa_x\lambda_{1, 0}}{\sigma^{2}}  \big)}{\lambda_{d, 0}^2\log\big(1 + \frac{\kappa_x \kappa_b \lambda_{1, 0}^2 \lambda_{1, \Psi}}{\sigma^{2} \lambda_{d, 0}^2}\big)}\,.&
\end{talign*}
In particular, the result in \cref{thm:regret} is retrieved when $\lambda_{1, 0}= \lambda_{d, 0}=\sigma_0^2\,, \text{ and }\lambda_{1, \Psi}=\sigma_\Psi^2\,.$
\end{theorem}

\begin{proof}

Consider our model rewritten in \eqref{eq:model_rewritten}. As we explained in \cref{sub:sketch}, the posterior distribution of the action parameter $\theta_{*,i} \mid H_t$ is a multivariate Gaussian distribution $\cN(\hat{\mu}_{t,i} , \hat{\Sigma}_{t,i})$ for some $\hat{\mu}_{t,i} \in \real^{d}$ and $\hat{\Sigma}_{t,i} \in \real^{d \times d}$ (\cref{lemma:gaussian_covariance}). Now we let $\theta_{t,*} = (X_t^\top \theta_{*, i})_{i \in [K]} \in \real^K$ be the concatenation of the expected rewards of actions in round $t$. Notice that the context $X_t$ is known in round $t$, and thus we include it in the history $H_t$. Then, the joint posterior of the expected rewards, $\theta_{t,*}  \mid H_t$, is also a multivariate Gaussian $\cN(\check{\theta}_t, \check{\Sigma}_t)$ for $\check{\theta}_t = (X_t^\top \hat{\mu}_{t, i})_{i \in [K]} \in \real^K$ and $\check{\Sigma}_t \in \real^{K \times K}$. This follows from the properties of Gaussian distributions \citep{koller09probabilistic} and the fact that $X_t$ is now included in $H_t$. Let $\bm{A}_{t} \in \{0, 1\}^K$ and $\bm{A}_{t, *} \in \{0, 1\}^K$ be indicator vectors of the taken action $A_t$ and optimal action $A_{t, *}$, respectively (The vector representations are in bold letters while the integer representations are in regular letters). Then the Bayes regret can be rewritten and consequently decomposed following standard analysis \citep{russo14learning} as 
\begin{align}\label{standard_decomposition_app} \mathcal{B}\mathcal{R}(n)
 &= \mathbb{E}\left[ \sum_{t = 1}^n X_t^\top \theta_{*, A_{t, *}} - X_t^\top \theta_{*, A_{t}}\right]\,,\\
 &= \mathbb{E}\Big[ \sum_{t = 1}^n \bm{A}_{t, *}^\top \theta_{t,*} - \bm{A}_{t}^\top \theta_{t,*}\Big]= \E{}{\condE{\bm{A}_{t, *}^\top (\theta_{t, *} - \check{\theta}_t)}{H_t}} +
  \E{}{\condE{\bm{A}_t^\top (\check{\theta}_t - \theta_{t, *})}{H_t}}\,.\nonumber
\end{align}
This follows from the fact that $\check{\theta}_t = (X_t^\top \hat{\mu}_{t, i})_{i \in [K]}$ is deterministic given $H_t$ (since $H_t$ now includes $X_t$), and that $\bm{A}_{t, *}$ and $\bm{A}_t$ are i.i.d. given $H_t$. Moreover, given $H_t$, $\check{\theta}_t - \theta_{t,*}$ is a zero-mean multivariate random variable independent of $\bm{A}_t$ and thus $\mathbb{E}[\bm{A}_t^\top (\check{\theta}_t - \theta_{t, *})\mid H_t]=0$. Therefore, we only need to bound the first term in \eqref{standard_decomposition_app}. With slight abuse of notation, let $\bm{\cA}$ be the set of all possible indicator vectors of actions $a \in [K]$. Precisely, an action $a \in [K]$ is also represented by an indicator vector $\bm{a} \in \bm{\cA} \subset \{0, 1\}^K$ (in bold letter). Then we define the following events
\begin{align*}
  E_{t, \bm{a}}(\delta) & =
  \set{|\bm{a}^\top (\theta_{t,*} - \check{\theta}_t)|
  \leq \sqrt{2 \log(1 / \delta)} \normw{\bm{a}}{\check{\Sigma}_t}}\,, & \forall \delta \in (0, 1)\,, \, \forall \bm{a} \in \bm{\cA}\,.
\end{align*}
Fix history $H_t$, we split the expectation over the two complementary events $  E_{t, \bm{A}_{t,*}}(\delta)$ and $ \bar{E}_{t, \bm{A}_{t,*}}(\delta)$, and use the Cauchy-Schwarz inequality to obtain
\begin{align}\label{eq:classic_regret_decomposition_2}
  \condE{\bm{A}_{t,*}^\top (\theta_{t,*} - \check{\theta}_t)}{H_t}
  \leq \sqrt{2 \log(1 / \delta)} \, \condE{\normw{\bm{A}_{t,*}}{\check{\Sigma}_t}}{H_t} +
  \condE{\bm{A}_{t,*}^\top (\theta_{t,*} - \check{\theta}_t)
  \I{\bar{E}_{t, \bm{A}_{t,*}}(\delta)}}{H_t}\,.
\end{align}
Now the second term in \eqref{eq:classic_regret_decomposition_2} can be bounded as follows. For any $\bm{a} \in \bm{\cA}\,,$ let $Z_{\bm{a}} = \bm{a}^\top (\theta_{t,*} - \check{\theta}_t)$. Then we have that  
\begin{align}\label{eq:classic_regret_decomposition_3}
  \condE{\bm{A}_{t,*}^\top (\theta_{t,*} - \check{\theta}_t) \I{ \bar{E}_{t, \bm{A}_{t,*}}(\delta) } }{H_t} & \stackrel{(i)}{=} \condE{Z_{\bm{A}_{t,*}} \I{  |Z_{\bm{A}_{t,*}}| > \sqrt{2 \log(1 / \delta)}\normw{\bm{A}_{t,*}}{\check{\Sigma}_t} } }{H_t}\,, \nonumber \\
& \stackrel{(ii)}{\leq} \condE{|Z_{\bm{A}_{t,*}}| \I{  |Z_{\bm{A}_{t,*}}| > \sqrt{2 \log(1 / \delta)}\normw{\bm{A}_{t,*}}{\check{\Sigma}_t} } }{H_t}\,, \nonumber \\
& \stackrel{(iii)}{\leq} \sum_{\bm{a} \in \bm{\cA}} \condE{|Z_{\bm{a}}| \I{  |Z_{\bm{a}}| > \sqrt{2 \log(1 / \delta)}\normw{\bm{a}}{\check{\Sigma}_t} } }{H_t}\,, \nonumber \\
  & \stackrel{(iv)}{\leq}  \sum_{\bm{a} \in \bm{\cA}}  \frac{2}{\normw{\bm{a}}{\check{\Sigma}_t}\sqrt{2 \pi}}
  \int_{u = \sqrt{2 \log(1 / \delta)} \normw{\bm{a}}{\check{\Sigma}_t}}^\infty
  u \exp\left[- \frac{u^2}{2\normw{\bm{a}}{\check{\Sigma}_t}^2}\right] \dif u\,, \nonumber \\
  & \stackrel{(v)}{\leq}  \sum_{\bm{a} \in \bm{\cA}} \normw{\bm{a}}{\check{\Sigma}_t} \frac{2}{\sqrt{2 \pi}}
  \int_{u = \sqrt{2 \log(1 / \delta)}}^\infty
  u \exp\left[- \frac{u^2}{2}\right] \dif u
   \stackrel{(vi)}{\leq} \sqrt{\frac{2}{\pi}} \lambda_{\max, t}   K \delta\,.
\end{align}
In $(i)$, we simply rewrite the terms using the random variable $Z_{\bm{A}_{t,*}}$. In $(ii)$, we use the fact that $Z_{\bm{A}_{t,*}} \leq |Z_{\bm{A}_{t,*}}|$. In $(iii)$, we upper bound the expectation of $|Z_{\bm{A}_{t,*}}| \I{  |Z_{\bm{A}_{t,*}}| > \sqrt{2 \log(1 / \delta)}\normw{\bm{A}_{t,*}}{\check{\Sigma}_t}}$ with the sum of the expectations of $|Z_{\bm{a}}| \I{  |Z_{\bm{a}}| > \sqrt{2 \log(1 / \delta)}\normw{\bm{a}}{\check{\Sigma}_t} }$ for $\bm{a} \in \bm{\cA}$ since all these random variables are non-negative. Moreover, $(iv)$ follows from the facts that given $H_t\,,$ $Z_{\bm{a}} \sim \cN(0, \normw{\bm{a}}{\check{\Sigma}_t}^2)$, and that if $Z \sim \cN(0, \sigma^2)$, then for any $\epsilon \geq 0\,,$ $\mathbb{P}(|Z|>\epsilon) \leq 2 \mathbb{P}(Z>\epsilon)$. In $(v)$, we use the change of variables $u \gets u/\normw{\bm{a}}{\check{\Sigma}_t} $. Finally, in $(vi)$, we compute the integral and set $\lambda_{\max, t} = \max_{\bm{a} \in \bm{\cA}} \normw{\bm{a}}{\check{\Sigma}_t}$. We combine \eqref{eq:classic_regret_decomposition_2} and \eqref{eq:classic_regret_decomposition_3} with the fact that $\bm{A}_t$ and $\bm{A}_{t,*}$ are i.i.d.\ given $H_t$ to obtain that
\begin{align}\label{eq:classic_regret_decomposition_4}
  \condE{\bm{A}_{t,*}^\top (\theta_{t,*} - \check{\theta}_t)}{H_t}
  \leq \sqrt{2 \log(1 / \delta)} \, \condE{\normw{\bm{A}_t}{\check{\Sigma}_t}}{H_t} +
  \sqrt{\frac{2}{\pi}} \lambda_{\max, t}   K \delta\,.
\end{align}
The bound in \eqref{eq:classic_regret_decomposition_4} holds for any history $H_t$ and thus we take an additional expectation and get that
\begin{align*}
 \mathcal{B}\mathcal{R}(n) = \E{}{\sum_{t = 1}^n \bm{A}_{t,*}^\top \theta_{t,*} - \bm{A}_t^\top \theta_{t,*}}
  & \leq \sqrt{2 \log(1 / \delta)} \,
  \E{}{\sum_{t = 1}^n \normw{\bm{A}_t}{\check{\Sigma}_t}} +
  \sqrt{\frac{2}{\pi}} \lambda_{\max, t}   K n \delta \,, \\
  & \stackrel{(i)}{\leq} \sqrt{2 n\log(1 / \delta)} \,
  \E{}{\sqrt{\sum_{t = 1}^n \normw{\bm{A}_t}{\check{\Sigma}_t}^2}} +
  \sqrt{\frac{2}{\pi}} \lambda_{\max, t}   K n \delta \,, \\
  & \stackrel{(ii)}{\leq} \sqrt{2 n \log(1 / \delta)}
  \sqrt{\E{}{\sum_{t = 1}^n \normw{\bm{A}_t}{\check{\Sigma}_t}^2}} +
  \sqrt{\frac{2}{\pi}} \lambda_{\max, t}   K n \delta\,,
\end{align*}
where we use the Cauchy-Schwarz inequality in $(i)$, and $(ii)$ follows from the concavity of the square root. Now note that any $\bm{a} \in \bm{\cA}$ is an indicator vector and that $\check{\Sigma}_t$ is the covariance of the joint posterior of the expected rewards $(X_t^\top \theta_{*, a})_{a \in [K] }\mid H_t$. Therefore, for  any $\bm{a} \in \bm{\cA}$, $\normw{\bm{a}}{\check{\Sigma}_t}^2 = \check{\sigma}_a^2$ is the variance of $X_t^\top \theta_{*, a} \mid H_t$. But we know that $\theta_{*, a} \mid H_t$ is a multivariate Gaussian and its covariance is $\hat{\Sigma}_{t,a}$ (\cref{lemma:gaussian_covariance}). Thus the variance of $X^\top \theta_{*, a} \mid H_t$ is $\check{\sigma}_a^2 = X_t^\top \hat{\Sigma}_{t,a}X_t$. It follows that for any 
$\bm{a} \in \bm{\cA}\,,$ $\normw{\bm{a}}{\check{\Sigma}_t}^2= X_t^\top \hat{\Sigma}_{t,a}X_t = \normw{X_t}{\hat{\Sigma}_{t, a}}^2$. In particular,  $\normw{\bm{A}_t}{\check{\Sigma}_t}^2= X_t^\top \hat{\Sigma}_{t,A_t}X_t$. Combining this with \eqref{eq:max_eigenvalue_covariance} yields that $\lambda_{\max, t}  = \max_{\bm{a} \in \bm{\cA}} \normw{\bm{a}}{\check{\Sigma}_t} = \max_{a \in \cA} \normw{X_t}{\hat{\Sigma}_{t, a}} \leq  \max_{a \in \cA}\sqrt{\lambda_1(\hat{\Sigma}_{t,a}) \kappa_x} \leq \sqrt{\left( \lambda_{1, 0} + \frac{\lambda_{1, 0}^2 \lambda_{1, \Psi}\kappa_b }{\lambda_{d, 0}^2}\right) \kappa_x}$. Then we let $c = \sqrt{ \frac{2}{\pi}\left( \lambda_{1, 0} + \frac{\lambda_{1, 0}^2 \lambda_{1, \Psi}\kappa_b }{\lambda_{d, 0}^2}\right) \kappa_x} K$ which allows us to write
\begin{align}
   \mkern-22mu \mathcal{B}\mathcal{R}(n) &\leq  \sqrt{2 n \log(1/\delta)}
  \sqrt{\E{}{\sum_{t = 1}^n \normw{X_t}{\hat{\Sigma}_{t,A_t}}^2}} +
  c n\delta\,.
\end{align}
Now we focus on the the term $\sqrt{\E{}{\sum_{t = 1}^n \normw{X_t}{\hat{\Sigma}_{t, A_t}}^2}}$ that we decompose and bound as
\begin{align}\label{eq:sequential proof decomposition}
  & \normw{X_t}{\hat{\Sigma}_{t, A_t}}^2 = \sigma^2 \frac{X_t^\top \hat{\Sigma}_{t, A_t} X_t}{\sigma^2} \stackrel{(i)}{=} \sigma^2 \left(\sigma^{-2} X_t^\top \tilde{\Sigma}_{t, A_t} X_t +
  \sigma^{-2} X_t^\top \tilde{\Sigma}_{t, A_t} \Sigma_{0,A_t}^{-1} \Gamma_{A_t} \bar{\Sigma}_t \Gamma_{A_t}^\top
  \Sigma_{0,A_t}^{-1} \tilde{\Sigma}_{t, A_t} X_t\right)\,,
  \nonumber \\
  & \stackrel{(ii)}{\leq} c_{\textsc{a}} \log(1 + \sigma^{-2} X_t^\top \tilde{\Sigma}_{t, A_t} X_t) +
  c_1 \log(1 + \sigma^{-2} X_t^\top \tilde{\Sigma}_{t, A_t} \Sigma_{0,A_t}^{-1} \Gamma_{A_t} \bar{\Sigma}_t \Gamma_{A_t}^\top
  \Sigma_{0,A_t}^{-1} \tilde{\Sigma}_{t, A_t} X_t)\,,
\end{align}
where $(i)$ follows from $\hat{\Sigma}_{t, A_t} =  \tilde{\Sigma}_{t, A_t}  + \tilde{\Sigma}_{t, A_t} \Sigma_{0,A_t}^{-1} \Gamma_{A_t} \bar{\Sigma}_t \Gamma_{A_t}^\top \Sigma_{0,A_t}^{-1}  \tilde{\Sigma}_{t, A_t}$, and we use the following inequality in $(ii)$
\begin{align*}
  x
  = \frac{x}{\log(1 + x)} \log(1 + x)
  \leq \left(\max_{x \in [0, u]} \frac{x}{\log(1 + x)}\right) \log(1 + x)
  = \frac{u}{\log(1 + u)} \log(1 + x)\,,
\end{align*}
which holds for any $x \in [0, u]$, where constants $c_{\textsc{a}}$ and $c_1$ are derived as
\begin{align*}
  c_{\textsc{a}}
  = \frac{\kappa_x \lambda_{1, 0}}{\log(1 +  \sigma^{-2} \kappa_x \lambda_{1, 0})}\,, \quad
  c_1
  = \frac{c_\Psi}{\log(1 + \sigma^{-2} c_\Psi)}\,, \quad
  c_\Psi
  = \frac{\kappa_x \kappa_b \lambda_{1, 0}^2 \lambda_{1, \Psi}}{\lambda_{d, 0}^2}\,,
\end{align*}
The derivation of $c_{\textsc{a}}$ uses that
\begin{align*}
  X_t^\top \tilde{\Sigma}_{t, A_t} X_t
  \leq \lambda_1(\tilde{\Sigma}_{t, A_t}) \norm{X_t}^2
  \leq  \lambda_d^{-1}(\Sigma_{0,A_t}^{-1} + G_{t, A_t}) \kappa_x
  \leq \lambda_d^{-1}(\Sigma_{0,A_t}^{-1}) \kappa_x
  = \lambda_1(\Sigma_{0,A_t}) \kappa_x \leq  \lambda_{1, 0} \kappa_x \,.
\end{align*}
The derivation of $c_1$ follows from
\begin{align*}
  X_t^\top \tilde{\Sigma}_{t, A_t} \Sigma_{0,A_t}^{-1} \Gamma_{A_t} \bar{\Sigma}_t \Gamma_{A_t}^\top \Sigma_{0,A_t}^{-1} \tilde{\Sigma}_{t, A_t} X_t
  \leq \lambda_1^2(\tilde{\Sigma}_{t, A_t}) \lambda_1^2(\Sigma_{0,A_t}^{-1}) \lambda_1(\Gamma_{A_t} \bar{\Sigma}_t \Gamma_{A_t}^\top) \kappa_x
  & \leq \frac{\lambda_1^2(\Sigma_{0,A_t}) \lambda_{1, \Psi}\lambda_1(\Gamma_{A_t} \Gamma_{A_t}^\top)  \kappa_x}{\lambda_d^2(\Sigma_{0,A_t})}\,, \\
  &\leq \frac{\lambda_{1, 0}^2 \lambda_{1, \Psi}\kappa_b \kappa_x}{\lambda_{d,0}^2}\,.
\end{align*}
The first inequality follows from Weyl's inequality and the fact that $\lambda_1(\bar{\Sigma}_t) \leq \lambda_1(\Sigma_\Psi) = \lambda_{1, \Psi}  $ and $\lambda_1(\tilde{\Sigma}_{t, A_t}) \leq \lambda_1(\Sigma_{0, A_t})$. Now we focus on bounding the logarithmic terms in \eqref{eq:sequential proof decomposition}.

\textbf{First Term in \eqref{eq:sequential proof decomposition}} We first rewrite this term as
\begin{align*}
   \log(1 + \sigma^{-2} X_t^\top \tilde{\Sigma}_{t, A_t} X_t) &\stackrel{(i)}{=} \log\det(I_d + \sigma^{-2}\tilde{\Sigma}_{t, A_t}^\frac{1}{2} X_t X_t^\top \tilde{\Sigma}_{t, A_t}^\frac{1}{2})\,,\\
  &= \log\det(\tilde{\Sigma}_{t, A_t}^{-1} + \sigma^{-2} X_t X_t^\top) - \log\det(\tilde{\Sigma}_{t, A_t}^{-1})
  = \log\det(\tilde{\Sigma}_{t+1, A_t}^{-1}) - \log\det(\tilde{\Sigma}_{t, A_t}^{-1})\,,
\end{align*}
where $(i)$ follows from the Weinstein–Aronszajn identity. Then we sum over all rounds $ t \in [n]$, and get a telescoping that leads to
\begin{align*}
  \sum_{t = 1}^{n}
   &\log\det(I_d +\sigma^{-2} \tilde{\Sigma}_{t, A_t}^\frac{1}{2} X_t
  X_t^\top \tilde{\Sigma}_{t, A_t}^\frac{1}{2})=\sum_{t = 1}^{n} \log\det(\tilde{\Sigma}_{t+1, A_t}^{-1}) - \log\det(\tilde{\Sigma}_{t, A_t}^{-1})\,,\\
  &=\sum_{t = 1}^{n} \sum_{i=1}^K \log\det(\tilde{\Sigma}_{t+1, i}^{-1}) - \log\det(\tilde{\Sigma}_{t, i}^{-1})=\sum_{i=1}^K \sum_{t = 1}^{n} \log\det(\tilde{\Sigma}_{t+1, i}^{-1}) - \log\det(\tilde{\Sigma}_{t, i}^{-1})\,,\\
  &= \sum_{i=1}^K \log\det(\tilde{\Sigma}_{n+1, i}^{-1}) - \log\det(\tilde{\Sigma}_{1, i}^{-1})
  \stackrel{(i)}{=} \sum_{i=1}^K \log\det(\Sigma_{0,i}^\frac{1}{2} \tilde{\Sigma}_{n+1, i}^{-1} \Sigma_{0,i}^\frac{1}{2}) \stackrel{(ii)}{\leq} \sum_{i=1}^K d \log\left(\frac{1}{d} \operatorname{Tr}(\Sigma_{0,i}^\frac{1}{2} \tilde{\Sigma}_{n+1,i}^{-1}
  \Sigma_{0,i}^\frac{1}{2})\right)\\
  &\leq \sum_{i=1}^K d \log\left(1 + \frac{\kappa_x \lambda_1(\Sigma_{0,i})  n}{\sigma^2 d}\right)\leq K d \log\left(1 + \frac{\kappa_x \lambda_{1, 0} n}{\sigma^2 d}\right)\,.
\end{align*}
where $(i)$ follows from the fact that $\tilde{\Sigma}_{1, i} = \Sigma_{0,i}$ and we use the inequality of arithmetic and geometric means in $(ii)$.

\textbf{Second Term in \eqref{eq:sequential proof decomposition}}
First, we rewrite the covariance matrix of the effect posterior $\bar{\Sigma}_t$ using the compact notation introduced in \cref{subsec:regre_proof_1}. Precisely, it follows from \eqref{eq:effect_posterior_formulas} that
\begin{align}\label{eq:hyper_posterior_cov_rewritten}
    \bar{\Sigma}_t^{-1} \stackrel{(i)}{=} \Sigma_\Psi^{-1} + \sum_{i=1}^K \Gamma_{i}^{\top} \left( \Sigma_{0,i} + G_{t,i}^{-1} \right)^{-1}\Gamma_{i} &\stackrel{(ii)}{=}  \Sigma_\Psi^{-1} + \sum_{i=1}^K\Gamma_i^{\top} \left( \Sigma_{0, i}^{-1} - \Sigma_{0, i}^{-1} (G_{t, i} + \Sigma_{0, i}^{-1})^{-1} \Sigma_{0, i}^{-1} \right)\Gamma_i \,,\nonumber\\
    &\stackrel{(iii)}{=}  \Sigma_\Psi^{-1} + \sum_{i=1}^K\Gamma_i^{\top} \left( \Sigma_{0, i}^{-1} - \Sigma_{0, i}^{-1} \tilde{\Sigma}_{t, i} \Sigma_{0, i}^{-1} \right)\Gamma_i \,.
\end{align}
The equality $(i)$ requires $G_{t,i}$ to be invertible and was only given in the main manuscript to ease the exposition. In our proof, we use $(ii)$ and $(iii)$ which are the same; $(iii)$ follows from plugging $\tilde{\Sigma}_{t, i} = (G_{t, i} + \Sigma_{0, i}^{-1})^{-1}$ in $(ii)$. Now let $u = \sigma^{-1} \tilde{\Sigma}_{t, A_t}^{\frac{1}{2}} X_t$. Then it follows from $(iii)$ in \eqref{eq:hyper_posterior_cov_rewritten} that
\begin{align}
  \bar{\Sigma}_{t + 1}^{-1} - \bar{\Sigma}_t^{-1} & = \Gamma_{A_t}^\top \left(\Sigma_{0,A_t}^{-1} - \Sigma_{0,A_t}^{-1} (\tilde{\Sigma}_{t, A_t}^{-1} + \sigma^{-2} X_t X_t^\top)^{-1} \Sigma_{0,A_t}^{-1} -
  (\Sigma_{0,A_t}^{-1} - \Sigma_{0,A_t}^{-1} \tilde{\Sigma}_{t, A_t} \Sigma_{0,A_t}^{-1})\right)\Gamma_{A_t}\,,
  \nonumber \\
  & = \Gamma_{A_t}^\top \left(\Sigma_{0,A_t}^{-1} (\tilde{\Sigma}_{t, A_t} - (\tilde{\Sigma}_{t, A_t}^{-1} + \sigma^{-2} X_t X_t^\top)^{-1}) \Sigma_{0,A_t}^{-1}\right)\Gamma_{A_t}\,,
  \nonumber \\
  & = \Gamma_{A_t}^\top \left(\Sigma_{0,A_t}^{-1} \tilde{\Sigma}_{t, A_t}^{\frac{1}{2}}
  (I_d - (I_d + \sigma^{-2} \tilde{\Sigma}_{t, A_t}^{\frac{1}{2}} X_t X_t^\top \tilde{\Sigma}_{t, A_t}^{\frac{1}{2}})^{-1})
  \tilde{\Sigma}_{t, A_t}^{\frac{1}{2}} \Sigma_{0,A_t}^{-1}\right)\Gamma_{A_t}\,,
  \nonumber \\
  & = \Gamma_{A_t}^\top \left(\Sigma_{0,A_t}^{-1} \tilde{\Sigma}_{t, A_t}^{\frac{1}{2}}
  (I_d - (I_d + u u^\top)^{-1})
  \tilde{\Sigma}_{t, A_t}^{\frac{1}{2}} \Sigma_{0,A_t}^{-1}\right)\Gamma_{A_t}\,,
  \nonumber \\
  & \stackrel{(i)}{=} \Gamma_{A_t}^\top \left(\Sigma_{0,A_t}^{-1} \tilde{\Sigma}_{t, A_t}^{\frac{1}{2}}
  \frac{u u^\top}{1 + u^\top u}
  \tilde{\Sigma}_{t, A_t}^{\frac{1}{2}} \Sigma_{0,A_t}^{-1}\right)\Gamma_{A_t} = \sigma^{-2} \Gamma_{A_t}^\top \left( \Sigma_{0,A_t}^{-1} \tilde{\Sigma}_{t, A_t}
  \frac{X_t X_t^\top}{1 + u^\top u}
  \tilde{\Sigma}_{t, A_t} \Sigma_{0,A_t}^{-1}\right)\Gamma_{A_t}\,.
  \label{eq:linear telescoping}
\end{align}
In $(i)$ we use the Sherman-Morrison formula. Moreover, we have that $\norm{X_t}^2 \leq \kappa_x$. Therefore,
\begin{align*}
  1 + u^\top u
  = 1 + \sigma^{-2} X_t^\top \tilde{\Sigma}_{t, A_t} X_t
  \leq 1 + \sigma^{-2} \kappa_x \lambda_1(\Sigma_{0,A_t}) \leq  1 + \sigma^{-2} \kappa_x \lambda_{1, 0} = c_2\,.
\end{align*}
This allows us to bound the second logarithmic term in \eqref{eq:sequential proof decomposition} as
\begin{align*}
&\log(1 + \sigma^{-2} X_t^\top \tilde{\Sigma}_{t, A_t} \Sigma_{0,A_t}^{-1} \Gamma_{A_t} \bar{\Sigma}_t \Gamma_{A_t}^\top
  \Sigma_{0,A_t}^{-1} \tilde{\Sigma}_{t, A_t} X_t) \,,\\   & \quad \stackrel{(i)}{\leq} c_2\log(1 + c_2^{-1} \sigma^{-2}  X_t^\top \tilde{\Sigma}_{t, A_t} \Sigma_{0,A_t}^{-1} \Gamma_{A_t} \bar{\Sigma}_t \Gamma_{A_t}^\top
  \Sigma_{0,A_t}^{-1} \tilde{\Sigma}_{t, A_t} X_t)\,, \\
  & \quad \stackrel{(ii)}{=} c_2 \log\det(I_{Ld} +
  c_2^{-1}\sigma^{-2} \bar{\Sigma}^\frac{1}{2}_t \Gamma_{A_t}^\top \Sigma_{0,A_t}^{-1} \tilde{\Sigma}_{t, A_t} X_t X_t^\top
  \tilde{\Sigma}_{t, A_t} \Sigma_{0,A_t}^{-1} \Gamma_{A_t} \bar{\Sigma}^\frac{1}{2}_t ) \,, \\
  & \quad \stackrel{(iii)}{=} c_2 \left[\log\det(\bar{\Sigma}_t^{-1} +
  c_2^{-1}\sigma^{-2} \Gamma_{A_t}^\top \Sigma_{0,A_t}^{-1} \tilde{\Sigma}_{t, A_t} X_t X_t^\top \tilde{\Sigma}_{t, A_t} \Sigma_{0,A_t}^{-1} \Gamma_{A_t}) -
  \log\det(\bar{\Sigma}_t^{-1})\right] \,, \\
  & \quad \stackrel{(iv)}{\leq} c_2 \left[\log\det(\bar{\Sigma}_t^{-1} +
  \sigma^{-2} \Gamma_{A_t}^\top \Sigma_{0,A_t}^{-1} \tilde{\Sigma}_{t, A_t} \frac{X_t X_t^\top}{1+u^\top u}  \tilde{\Sigma}_{t, A_t} \Sigma_{0,A_t}^{-1} \Gamma_{A_t} ) -
  \log\det(\bar{\Sigma}_t^{-1})\right] \,, \\
  & \quad \stackrel{(v)}{=} c_2 \left[\log\det(\bar{\Sigma}_{t + 1}^{-1}) -
  \log\det(\bar{\Sigma}_t^{-1})\right]\,.
\end{align*}
Here $(i)$ follows from the fact that $\log(1 + x) \leq c_2 \log(1 + c_2^{-1}x)$ for any $x \geq 0$ and $c_2 \geq 1$. In $(ii)$, we use the Weinstein–Aronszajn identity. In $(iii)$, we use the $\log$ product formula and the fact that the $\det$ is a multiplicative map. In $(iv)$, we use that $ c_2^{-1} \leq 1 / (1 + u^\top u)$. Finally, $(v)$ follows from \eqref{eq:linear telescoping}. Now we sum over all rounds and get telescoping
\begin{align*}
   \sum_{t = 1}^{n} 
  &\log\det(I_{Ld} +  \sigma^{-2}
  \bar{\Sigma}^\frac{1}{2}_t  \Gamma_{A_t}^\top \Sigma_{0,A_t}^{-1} \tilde{\Sigma}_{t, A_t}  X_t X_t^\top
  \tilde{\Sigma}_{t, A_t} \Sigma_{0,A_t}^{-1} \bar{\Sigma}^\frac{1}{2}_t \Gamma_{A_t})\,,\\
  &\leq c_2 \left[\log\det(\bar{\Sigma}_{n+1}^{-1}) -
  \log\det(\bar{\Sigma}_1^{-1})\right] = c_2 \log\det(\Sigma_\Psi^\frac{1}{2} \bar{\Sigma}_{n + 1}^{-1} \Sigma_\Psi^\frac{1}{2})   \stackrel{(i)}{\leq} c_2 L d \log\left(\frac{1}{Ld} \operatorname{Tr}(\Sigma_\Psi^\frac{1}{2} \bar{\Sigma}_{n+1}^{-1}
  \Sigma_\Psi^\frac{1}{2})\right) \,,\\
  &  \stackrel{(ii)}{\leq} c_2 L d
  \log(\lambda_1(\Sigma_\Psi^\frac{1}{2} \bar{\Sigma}_{n+1}^{-1} \Sigma_\Psi^\frac{1}{2}))   \stackrel{(iii)}{\leq} c_2 L d \log\Big(1+K \kappa_b \lambda_{1, \Psi} \big( \frac{1}{\lambda_{d, 0}} - \frac{1}{\lambda_{1, 0}^2\big(\frac{\kappa_x n}{\sigma^2} + \frac{1}{\lambda_{d, 0}} \big)} \big)\Big)\,,\\
\end{align*}
In $(i)$ we use the inequality of arithmetic and geometric means. In $(ii)$ we bound all eigenvalues in the trace by the maximum eigenvalue. In $(iii)$ we use the result in \eqref{eq:max_eigenvalue_effect_covariance}. We combine the upper bounds for both logarithmic terms and get
\begin{align*}
  \E{}{\sum_{t = 1}^n \normw{X_t}{\hat{\Sigma}_{t, A_t}}^2}
  \leq K d c_{\textsc{a}} \log\big(1 + \frac{\kappa_x \lambda_{1, 0}  n}{\sigma^2 d}\big) +
  Ld c_1 c_2 \log\Big(1+K \kappa_b \lambda_{1, \Psi} \big( \frac{1}{\lambda_{d, 0}} - \frac{1}{\lambda_{1, 0}^2\big(\frac{\kappa_x n}{\sigma^2} + \frac{1}{\lambda_{d, 0}} \big)} \big)\Big)\,.
\end{align*}
Finally, we set $c_{\textsc{e}} = c_1 c_2$,  which concludes the proof for the general case. To retrieve the result in \cref{thm:regret}, we only need to set $\lambda_{1, 0} = \lambda_{d, 0} = \sigma_0^2$ and $\lambda_{1, \Psi} = \sigma_\Psi^2$ since we assumed that $\Sigma_\Psi = \sigma_\Psi^2 I_{Ld}$ and that $\Sigma_{0, i} = \sigma_0^2 I_d$ for any $i \in [K]$. In that case, the second term simplifies as
\begin{align*}
     \log\Big(1+K \kappa_b \lambda_{1, \Psi} \big( \frac{1}{\lambda_{d, 0}} - \frac{1}{\lambda_{1, 0}^2\big(\frac{\kappa_x n}{\sigma^2} + \frac{1}{\lambda_{d, 0}} \big)} \big)\Big)& =  \log\Big(1+K \kappa_b \sigma_\Psi^2 \big( \frac{1}{\sigma_{0}^2} - \frac{1}{\sigma_{0}^4\big(\frac{\kappa_x n}{\sigma^2} + \frac{1}{\sigma_{0}^2} \big)} \big)\Big)\,,\\
     & =  \log\big(1+K \kappa_b \sigma_\Psi^2 \frac{n\kappa_x}{ n \kappa_x \sigma_{0}^2  + \sigma^2} \big)\,.
\end{align*}
\end{proof}

\section{EXTENSIONS}\label{app:extensions}
Here we present and discuss in detail possible extensions of \alg. We start with the factored approximation of the effect posteriors (\cref{app:posterior_approximation}) which improves computational efficiency with minimal impact on the empirical regret (\cref{sec:experiments,app:experiments}). We provide closed-form solutions for the factored effect posteriors in all the settings that we consider in this paper. While \cref{alg:ts} can be applied to the general two-level hierarchical setting introduced in \cref{sec:setting}, we only focused on cases where the dependencies of action parameters with effect parameters are captured through a linear combination in the theoretical analysis and experiments. In \cref{subsec:mixed_model}, we provide an extension of our analysis to the case where the weights $b_{i, \ell}$ are replaced by matrices $\C_{i, \ell}$. Moreover, in \cref{subsec:non_linear_mixed_model}, we present a way to introduce non-linearity in effects. In \cref{subsec:multi_level}, we motivate \emph{deeper} hierarchies, and provide intuition on the corresponding regret.

\subsection{Factored Effect Posteriors}\label{app:posterior_approximation}

As discussed earlier, the number of actions $K$ is often much larger than the number of effect parameters $L$. However, $L$ can also be large. In this section, we show how to improve the computational efficiency of \alg using factored distributions \citep{Bishop2006}. Consider the practical models in \eqref{eq:contextual_gaussian_model} and \eqref{eq:contextual_bernoulli_model} where the effect posterior is a multivariate Gaussian $Q_{t} = \cN(\bar{\mu}_{t}, \bar{\Sigma}_{t})$ (\cref{sec:linear bandit posterior,sec:glb bandit posterior}). Now suppose that it factorizes, that is $Q_t(\Psi) = \prod_{\ell=1}^L Q_{t, \ell}(\psi_\ell)$, where $Q_{t, \ell}$ is the effect posterior of the $\ell$-th effect parameter $\psi_{*, \ell}$. Then for any round $t \in [n]$, the effect posterior $Q_{t, \ell}$ is also a multivariate Gaussian $Q_{t, \ell} = \cN(\bar{\mu}_{t, \ell}, \bar{\Lambda}_{t, \ell}^{-1})$, where $\bar{\Lambda}_{t, \ell}$ is the $\ell$-th $d \times d$ diagonal block of $\bar{\Sigma}_{t}^{-1}$, and $\bar{\mu}_{t, \ell} \in \real^d$ are such that $\bar{\mu}_{t} = (\bar{\mu}_{t, \ell})_{\ell \in [L]}$. This allows for individual sampling of the effect parameters, which improves the space and time complexity. Next we provide the factored effect posterior for the mixed-effect bandit settings considered in our paper.

\textbf{Mixed-Effect Linear Bandit} Consider the model in \eqref{eq:contextual_gaussian_model}, we have that for any round $t \in [n]$, the $\ell$-th effect posterior $Q_{t, \ell}$ is also a multivariate Gaussian $Q_{t, \ell} = \cN(\bar{\mu}_{t, \ell}, \bar{\Sigma}_{t, \ell})$, where
\begin{talign}
    &\bar{\Sigma}_{t, \ell}^{-1}= \Sigma_{\psi_\ell}^{-1} +  \sum_{i \in [K]} b_{i, \ell}^2 \left( \Sigma_{0, i} + G_{t, i}^{-1} \right)^{-1}\,,\nonumber\\
    \bar{\mu}_{t, \ell} &= \bar{\Sigma}_{t, \ell} \left( \Sigma_{\psi_\ell}^{-1}  \mu_{\psi_\ell} +\sum_{i \in [K]}   b_{i, \ell} \left( (\Sigma_{0, i} + G_{t, i}^{-1})^{-1} G_{t, i}^{-1} B_{t, i}\right)\right)\,.
\end{talign}
$\Sigma_{\psi_\ell}$ is the $\ell$-th $d \times d$ diagonal block of $\Sigma_\Psi$, and  $\mu_{\psi_\ell} \in \real^d$ are such that $\mu_\Psi = (\mu_{\psi_\ell})_{\ell \in [L]}$.

\textbf{Mixed-Effect Generalized Linear Bandit} Consider the model in \eqref{eq:contextual_bernoulli_model}, we have that for any round $t \in [n]$, the $\ell$-th effect posterior $Q_{t, \ell}$ is also a multivariate Gaussian $Q_{t, \ell} = \cN(\bar{\mu}_{t, \ell}, \bar{\Sigma}_{t, \ell})$, where
\begin{talign}
    &\bar{\Sigma}_{t, \ell}^{-1}= \Sigma_{\psi_\ell}^{-1} +  \sum_{i \in [K]} b_{i, \ell}^2 \left( \Sigma_{0, i} + (G^{\textsc{lap}}_{t, i})^{-1} \right)^{-1}\,,\nonumber\\
    \bar{\mu}_{t, \ell} &= \bar{\Sigma}_{t, \ell} \left( \Sigma_{\psi_\ell}^{-1}  \mu_{\psi_\ell} +\sum_{i \in [K]}   b_{i, \ell} \left( (\Sigma_{0, i} + (G^{\textsc{lap}}_{t, i})^{-1})^{-1} \mu^\textsc{lap}_{t, i}\right)\right)\,.
\end{talign}
$\Sigma_{\psi_\ell}$ is the $\ell$-th $d \times d$ diagonal block of $\Sigma_\Psi$, and  $\mu_{\psi_\ell} \in \real^d$ are such that $\mu_\Psi = (\mu_{\psi_\ell})_{\ell \in [L]}$.

\textbf{Mixed-Effect Multi-Armed Bandit} Consider the model in \eqref{eq:gaussian_linear_bandits_model}, we have that  for any round $t \in [n]$, the effect posterior $Q_{t, \ell}$ is a \emph{univariate} Gaussian $Q_{t, \ell} = \cN(\bar{\mu}_{t, \ell}, \bar{\sigma}_{t, \ell}^2)$, where
\begin{align}
    &\bar{\sigma}_{t, \ell}^{-2}= \sigma_{\psi_\ell}^{-2} + \sum_{i\in[K]} b_{i, \ell}^2 \frac{N_{t, i}}{N_{t, i} \sigma_{0, i}^2 + \sigma^2}\,,\nonumber\\
    \bar{\mu}_{t, \ell} &= \bar{\sigma}_{t, \ell}^2 \left( \sigma_{\psi_\ell}^{-2}  \mu_{\psi_\ell} + \sum_{i\in[K]} b_{i, \ell} \frac{B_{t, i}}{N_{t, i}\sigma_{0, i}^2+\sigma^2}\right)\,.
\end{align}
$\sigma_{\psi_\ell}^2 > 0$ is the $\ell$-th diagonal entry of $\Sigma_\Psi$, and $\mu_{\psi_\ell} \in \real$ is the $\ell$-th entry of $\mu_\Psi$.

\begin{proof}

To reduce clutter, we consider a fixed round $t \in [n]$, and drop subindexing by $t$. It follows that $Q=\cN(\bar{\mu}, \bar{\Sigma})$ corresponds to the effect posterior $Q_t=\cN(\bar{\mu}_t, \bar{\Sigma}_t)$ for some round $t$. Here we restrict the family of effect posteriors $Q$ to factored distributions. Precisely, we first partition the elements of $\Psi_*=(\psi_{*, \ell})_{\ell \in [L]}$ into $L$ disjoint $d$-dimensional groups where each group corresponds to an effect parameter $\psi_{*, \ell}$.  We then suppose that $Q$ factorizes across the $L$ effect parameters, that is $Q(\Psi) = \prod_{\ell \in [L]} Q_\ell(\psi_\ell),$ where $Q_\ell$ are obtained using variational inference techniques \citep{Bishop2006} as we show next. First, we know that $Q(\Psi) = \cN(\bar{\mu}, \bar{\Sigma})$, where $\bar{\mu} \in \real^{Ld}$ and $\bar{\Sigma} \in \real^{Ld \times Ld}$. We write the mean and covariance by blocks as $\bar{\mu} = (\bar{\mu}_\ell)_{\ell \in [L]}$ and $\bar{\Sigma} = (\bar{\Sigma}_{i, j})_{(i, j) \in [L] \times [L]}$, such that $\bar{\mu}_\ell \in \real^d$ and $\bar{\Sigma}_{i, j} \in \real^{d \times d}$. Now fix $\ell \in [L]$, from know results \citep{Bishop2006} the optimal factor $Q_\ell$ that optimizes the Kullback-Leibler divergence satisfies
\begin{align}\label{eq:var_inference_1}
    Q_\ell(\psi_\ell) \propto \exp{\left(\E{j \neq \ell}{\log Q(\Psi)}\right)}\,,
\end{align}
where $\E{j \neq \ell}{\cdot}$ denotes an expectation with respect to the distributions $Q_j$ such that $j \neq \ell$. Let $\bar{\Lambda} = \bar{\Sigma}^{-1} =  (\bar{\Lambda}_{i, j})_{(i, j) \in [L] \times [L]}$, the expectation can be computed as
\begin{align}\label{eq:var_inference_2}
    Q_\ell(\psi_\ell) &\propto \exp{\left(\E{j \neq \ell}{ -\frac{1}{2} (\psi_\ell - \bar{\mu}_\ell)^\top \bar{\Lambda}_{\ell, \ell}  (\psi_\ell - \bar{\mu}_\ell) + \sum_{j \neq \ell} (\psi_\ell - \bar{\mu}_\ell)^\top \bar{\Lambda}_{\ell, j}  (\psi_j - \bar{\mu}_j)}\right)}\,,\nonumber\\
   &\propto \exp{\left(\E{j \neq \ell}{ -\frac{1}{2} \psi_\ell^\top  \bar{\Lambda}_{\ell, \ell} \psi_\ell + \psi_\ell^\top \bar{\Lambda}_{\ell, \ell} \bar{\mu}_\ell - \psi_\ell^\top \sum_{j \neq \ell} \bar{\Lambda}_{\ell, j} (\E{}{\psi_j} -  \bar{\mu}_j)}\right)}\,,\nonumber\\
   &\propto \exp{\left(\E{j \neq \ell}{ -\frac{1}{2} \psi_\ell^\top  \bar{\Lambda}_{\ell, \ell} \psi_\ell + \psi_\ell^\top \left( \bar{\Lambda}_{\ell, \ell} \bar{\mu}_\ell -  \sum_{j \neq \ell} \bar{\Lambda}_{\ell, j} (\E{}{\psi_j} -  \bar{\mu}_j)\right)}\right)}\,.
\end{align}
Thus, we have that 
\begin{align}\label{eq:var_inference_3}
    Q_\ell(\psi_\ell) & = \cN\left(\psi_\ell; m_\ell, \bar{\Lambda}_{\ell, \ell}^{-1}  \right)\,,
\end{align}
where $m_\ell = \bar{\mu}_\ell - \bar{\Lambda}_{\ell, \ell}^{-1} \sum_{j \neq \ell}  \bar{\Lambda}_{\ell, j} (\E{}{\psi_j} -  \bar{\mu}_j)$. These solutions are coupled since the optimal factor $Q_\ell$ depends on the other optimal factors $Q_j$ for $j \neq \ell$. However, we can provide a non-coupled solution in this Gaussian case if we set $m_\ell = \E{}{\psi_\ell} = \bar{\mu}_\ell$ for all $\ell \in [L]$; in which case we get that $Q_\ell(\psi_\ell) = \cN\left(\psi_\ell; \bar{\mu}_\ell, \bar{\Lambda}_{\ell, \ell}^{-1} \right)$ for all $\ell \in [L]$. To summarize, we showed that if we suppose that the effect posterior factorizes, that is $Q(\Psi) = \prod_{\ell \in [L]} Q_\ell(\psi_\ell)$, and that $Q= \cN(\bar{\mu}, \bar{\Lambda}^{-1})$, then the optimal factors $Q_\ell$ are also Gaussians $Q_\ell = \cN\left(\bar{\mu}_\ell, \bar{\Lambda}_{\ell, \ell}^{-1} \right)$, where $\bar{\mu}_\ell \in \real^d$ and  $ \bar{\Lambda}_{\ell, \ell} \in \real^{d \times d}$ are such that $\bar{\mu} = (\bar{\mu}_\ell)_{\ell \in [L]}$ and $ \bar{\Lambda} =  (\bar{\Lambda}_{i, j})_{(i, j) \in [L]\times [L]}$. Finally, to get the desired results, we simply retrieve the respective $\bar{\mu}_\ell \in \real^d$ and  $ \bar{\Lambda}_{\ell, \ell} \in \real^{d \times d}$ from the mean and inverse covariance of the exact posterior of either the model in \eqref{eq:contextual_gaussian_model}, \eqref{eq:contextual_bernoulli_model} or \eqref{eq:gaussian_linear_bandits_model}.
\end{proof}

\subsection{Finer Linear Effects} \label{subsec:mixed_model}
An effective way to capture fine-grained dependencies is to assume that the parameter of action $i$ depends on effect parameters through $L$ \emph{known} matrices $\C_{i, \ell} \in \real^{d \times d}$ as 
\begin{align*}
    \theta_{*, i} \mid \Psi_* \sim P_{0, i}\Big(\cdot \mid \sum_{\ell=1}^L \C_{i, \ell} \psi_{*, \ell}\Big)\,.
\end{align*}
This generalizes the setting considered in our analysis (\cref{subsec:contextual_gaussian_bandits}), which corresponds to the case where $\C_{i, \ell} = b_{i, \ell} I_d$. We first make the observation that $\sum_{\ell=1}^L \C_{i, \ell} \psi_{*, \ell} = \C_i \Psi_*$, where $\C_i = [\C_{i, 1}, \ldots, \C_{i, L}] \in \real^{d \times Ld}$. It follows that for any round $t \in [n]$, the joint effect posterior is a multivariate Gaussian $Q_t = \cN(\bar{\mu}_t, \bar{\Sigma}_t)$, where
\begin{align}
    \bar{\Sigma}_t^{-1}&= \Sigma_{\Psi}^{-1} +  \sum_{i=1}^K \C_i^\top \left( \Sigma_{0, i} + G_{t, i}^{-1} \right)^{-1} \C_i\,,\nonumber\\
    \bar{\mu}_t &= \bar{\Sigma}_t \Big( \Sigma_{\Psi}^{-1}  \mu_{\Psi} +\sum_{i=1}^K  \C_i^\top \left( (\Sigma_{0, i} + G_{t, i}^{-1})^{-1} G_{t, i}^{-1} B_{t, i}\right)\Big)\,.
\end{align}
Moreover, for any round $t \in [n]$ and action $i \in [K]$, the action posterior is a multivariate Gaussian $P_{t, i}(\cdot \mid \Psi_t) = \cN(\cdot;\tilde{\mu}_{t, i}, \tilde{\Sigma}_{t, i})$, where
\begin{align}
    \tilde{\Sigma}_{t, i}^{-1} &=  G_{t, i} + \Sigma_{0, i}^{-1}\,, \qquad
    \tilde{\mu}_{t, i} = \tilde{\Sigma}_{t, i} \Big( B_{t, i} +  \Sigma_{0, i}^{-1} \Big(\sum_{\ell=1}^L \C_{i, \ell} \psi_{t, \ell}\Big)\Big)\,.
\end{align}
Finally, our regret proof extends smoothly leading to a Bayes regret upper bound similar to the one we derived in \cref{thm:regret}. The corresponding Bayes regret is given in the following proposition. 

\begin{proposition}
\label{prop:regret_extension} For any $\delta \in (0, 1)$, the Bayes regret of \emph{\alg}, for the mixed-effect model in \cref{subsec:mixed_model}, is bounded as
\begin{align*}
  \mathcal{B}\mathcal{R}(n)
  \leq \sqrt{2 n 
  \left( \mathcal{R}^{\textsc{a}}(n) + \mathcal{R}^{\textsc{e}}(n) \right) \log(1 / \delta)} +
  c n\delta\,,
\end{align*}
where $c = \sqrt{\frac{2}{\pi} \kappa_x(\sigma_0^2 + \kappa_c \sigma_\Psi^2 )}K\,, \ \kappa_c = \max_{i \in [K]}\lambda_1(C_i^\top C_i)\,,$
\begin{talign*}
& \mathcal{R}^{\textsc{a}}(n) = dK c_\textsc{a} \log\big(1 + \frac{n\kappa_x\sigma_0^2}{d \sigma^2} \big)\,, \, c_\textsc{a} = \frac{ \kappa_x\sigma^2_{0}}{\log\big(1 + \frac{\kappa_x\sigma_0^2}{\sigma^{2}}\big)}\,,\\ 
&\mathcal{R}^{\textsc{e}}(n) = dL c_\textsc{e} \log\big(1 +  \frac{K \kappa_c \sigma^2_{\Psi} }{\sigma^2_{0} + \frac{\sigma^2}{n \kappa_x}}\big)\,,  \, c_\textsc{e}= \frac{\kappa_x \kappa_c \sigma_\Psi^2 \big(1 + \frac{\kappa_x\sigma_0^2}{\sigma^{2}}  \big)}{\log\big(1 + \frac{\kappa_x \kappa_c \sigma_\Psi^2}{\sigma^{2}}\big)}\,.&
\end{talign*}
\end{proposition}
The interpretation of this result is similar to \cref{thm:regret}. The only difference is that sparsity is now captured through $\kappa_{c}$.

\subsection{Non-Linear Effects} \label{subsec:non_linear_mixed_model}

Here the dependence of effect and action parameters are generated as in \eqref{eq:contextual_gaussian_model}, except that a non-linear function $g(\cdot)$ is applied to the linear combination $\sum_{\ell =1}^L b_{i, \ell} \psi_{*, \ell}$. An example of $g$ is the sigmoid function, and the whole model is
\begin{align}\label{eq:contextual_non_linear_effects}
    \Psi_{*} & \sim \cN(\mu_{\Psi}, \Sigma_{\Psi})\,,\\ 
    \theta_{*, i} \mid  \Psi_{*}& \sim \cN\Big( g\Big(\sum_{\ell =1}^L b_{i, \ell} \psi_{*, \ell}\Big) , \,  \Sigma_{0, i}\Big)\,, & \forall i \in [K]\,,\nonumber\\ 
    Y_t  \mid X_t, \theta_{*, A_t} & \sim \cN(X_t^\top \theta_{*, A_t} ,  \sigma^2)\,, & \forall t \in [n]\,. \nonumber
\end{align} 
The action posteriors has closed-form solution. Precisely, for any round $t \in [n]$, action $i \in [K]$, and effect parameters $\Psi_t$, the action posterior is a multivariate Gaussian $P_{t, i}(\cdot \mid \Psi_t) = \cN(\cdot;\tilde{\mu}_{t, i}, \tilde{\Sigma}_{t, i})$, where
\begin{align} 
   \tilde{\Sigma}_{t, i}^{-1}  &= \Sigma_{0, i}^{-1} + G_{t, i}\,, \\
    \tilde{\mu}_{t, i} &= \tilde{\Sigma}_{t, i} \Big( \Sigma_{0, i}^{-1} g\Big(\sum_{\ell =1}^L b_{i, \ell} \psi_{t, \ell}\Big) + B_{t, i}  \Big)\,.\nonumber
\end{align} 
The action posterior is the same as in \cref{thm:pti_gaussian} except that the prior term $\sum_{\ell =1}^Lb_{i, \ell} \psi_{t, \ell}$ is now replaced by $g\big(\sum_{\ell =1}^Lb_{i, \ell} \psi_{t, \ell}\big)$. The effect posterior does not have a closed-form solution and can be approximated using the Laplace approximation similarly to \cref{sec:glb bandit posterior}.

\subsection{Beyond Two-Level Hierarchies} \label{subsec:multi_level}
To motivate deeper hierarchies \citep{hong22deep}, consider the problem of page construction in movie streaming services where $J$ movies are organized into $L$ categories. First, a category $\ell \in [L]$ is associated with a parameter $\psi_{*, \ell} \in \real^d$. Moreover, each movie $j \in [J]$ is associated with a parameter $\phi_{*, j} \in \real^d$, which is a combination of category parameters $\psi_{*, \ell}$ weighted by scalars that quantify how related is the movie $j$ to each category. Finally, page layouts are actions and they are seen as lists (or slates) of movies. Each page layout $i \in [K]$ is associated with an action parameter $\theta_{*, i}$ which is also a combination of movies parameters $\phi_{*, j}$ weighted by a scalar that quantifies position bias. Precisely, this scalar is set to 0 if the corresponding movie is not present in the page, and it has high value if the movie is placed in a position with high visibility. This setting induces a three-level hierarchical model, for which we give a Gaussian example below.
\begin{align}\label{eq:multi_level_model}
    \Psi_{*} & \sim \cN(\mu_{\Psi}, \Sigma_{\Psi})\,, & \nonumber\\ 
    \phi_{*, j} \mid  \Psi_{*} & \sim \cN\left( \sum_{\ell=1}^L b_{j, \ell} \psi_{*, \ell} , \,  \Sigma_{\phi, j}\right)\,, & \forall j \in [J]\,, \nonumber\\ 
    \theta_{*, i} \mid   \Phi_{*}& \sim \cN\left( \sum_{j=1}^J w_{i, j} \phi_{*, j} , \,  \Sigma_{0, i}\right)\,, & \forall i \in [K]\,, \nonumber\\ 
    Y_t  \mid X_t, \theta_{*, A_t} & \sim \cN(X_t^\top \theta_{*, A_t} ,  \sigma^2)\,, & \forall t \in [n]\,,
\end{align}
where $\Psi_*= (\psi_{*, \ell})_{\ell \in [L]} \in \real^{Ld}$ and $\Phi_*= (\phi_{*, j})_{j \in [J]} \in \real^{Jd}$. \alg samples hierarchically as follows. First, we sample $\Psi_t$ from the posterior of $\Psi_* \mid H_t$. We then sample $\Phi_{t}$ from the posterior of $\Phi_{*} \mid \Psi_*, H_t$. Finally, we sample $\theta_{t, i}$ individually from the posterior of $\theta_{*, i} \mid \Phi_*, H_{t, i}$. We expect the upper bound of the Bayes regret of \eqref{eq:multi_level_model} following our analysis to be decomposed in three terms $\tilde{\mathcal{O}}\Big(\sqrt{n(\mathcal{R}^{\textsc{a}}(n)  + \mathcal{R}^{\textsc{e}}_1(n) + \mathcal{R}^{\textsc{e}}_2(n))}\Big)\,,$ where $\mathcal{R}^{\textsc{a}}(n)  = \tilde{\mathcal{O}}(Kd)$, $\mathcal{R}^{\textsc{e}}_1(n) = \tilde{\mathcal{O}}(Jd)$, and $\mathcal{R}^{\textsc{e}}_2(n) = \tilde{\mathcal{O}}(Ld)$.

\section{ADDITIONAL EXPERIMENTS}
\label{app:experiments}

We provide additional experiments where we evaluate \alg using synthetic and real-world problems, and compare it to baselines that either ignore or partially use effect parameters. In each plot, we report the averages and standard errors of the quantities. Both settings are described in \cref{sec:experiments}. 



\subsection{Synthetic Experiments}\label{app:synthetic}

In \cref{fig:app_synthetic_regret_lin,fig:app_synthetic_regret_log}, we report regret from 12 experiments with horizon $n=5000$, where we vary $K$ and $d$ and use both linear and logistic rewards. For the linear setting, we compare \alglin (\cref{sec:linear bandit posterior}), \linucb \citep{abbasi-yadkori11improved}, \lints \citep{agrawal13thompson} and \hierts \citep{hong22hierarchical}. For the logistic setting, we compare \algglm (\cref{sec:glb bandit posterior}), \alglin (\cref{sec:linear bandit posterior}), \ucbglm \citep{li17provably}, \glmts \citep{chapelle11empirical} and \hierts \citep{hong22hierarchical}. We also include the factored approximation of \alg (\alglinfa and \algglmfa). In all experiments, we observe that \alglin and \algfa outperform other baselines that ignore the effect parameters or incorporate them partially. We also notice that the gain in performance becomes smaller when $K/L$ decreases.

\begin{figure}[t!]
  \centering
  \includegraphics[width=\linewidth]{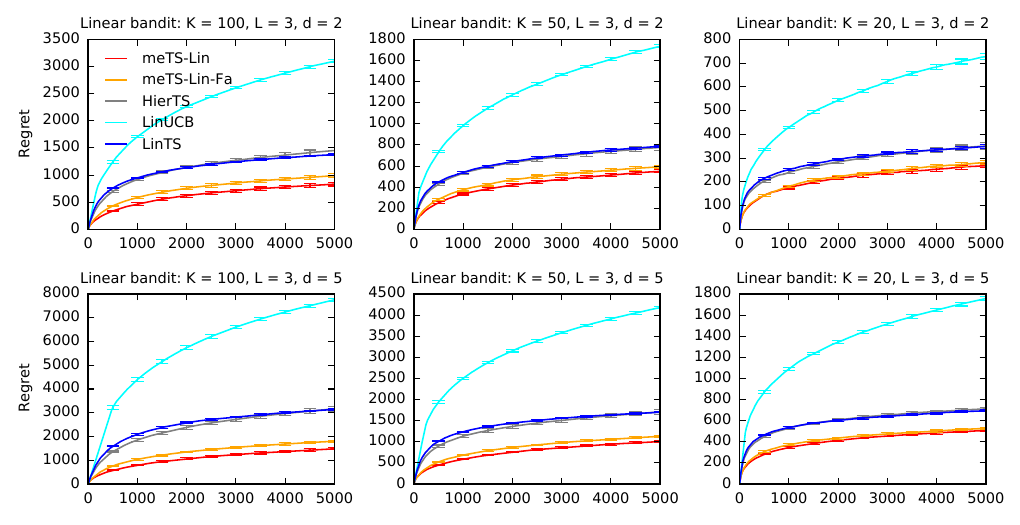}
  \caption{Regret of \alglin on synthetic linear bandit problems with varying feature dimension $d \in \{2, 5\}$ and number of actions $K\in \{20, 50, 100\}$.} 
  \label{fig:app_synthetic_regret_lin}
\end{figure}

\begin{figure}[t!]
  \centering
  \includegraphics[width=\linewidth]{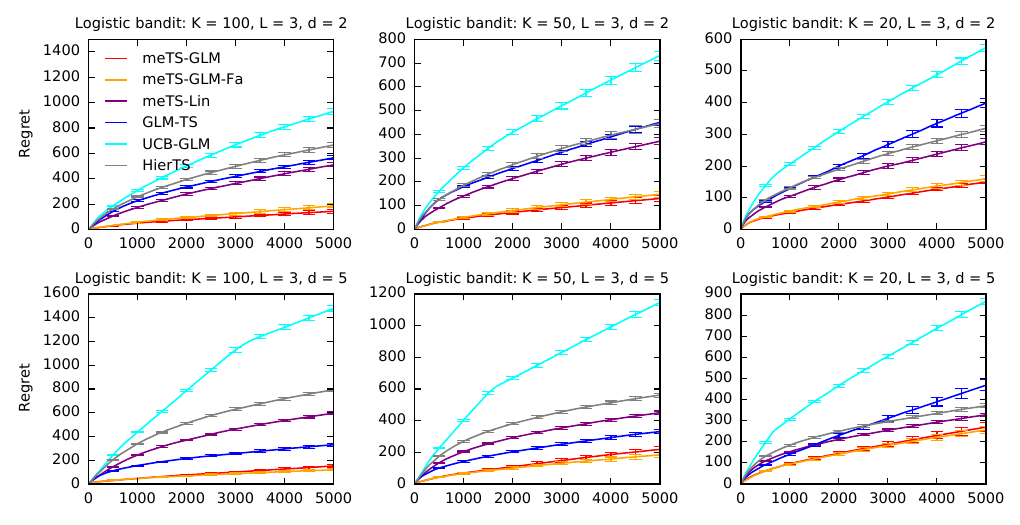}
  \caption{Regret of \algglm on synthetic logistic bandit problems with varying feature dimension $d \in \{2, 5\}$ and number of actions $K\in \{20, 50, 100\}$.} 
  \label{fig:app_synthetic_regret_log}
\end{figure}

\subsection{MovieLens Experiments}\label{app:movielens}

\begin{figure}[t!]
  \centering
  \includegraphics[width=\linewidth]{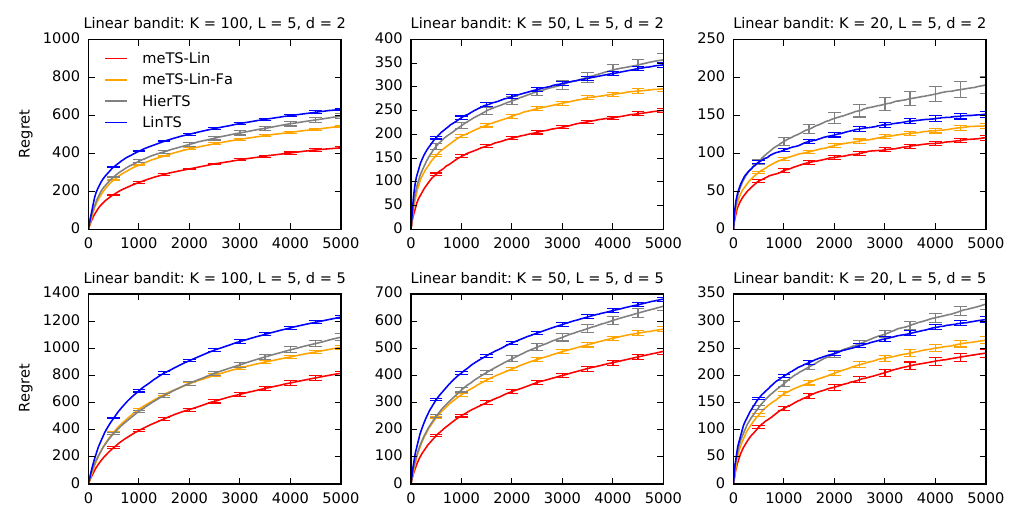}
  \caption{Regret of \alglin on the MovieLens dataset with linear rewards and varying feature dimension $d \in \{2, 5\}$ and number of actions $K\in \{20, 50, 100\}$.} 
  \label{fig:app_movielens_regret_lin}
\end{figure}

\begin{figure}[t!]
  \centering
  \includegraphics[width=\linewidth]{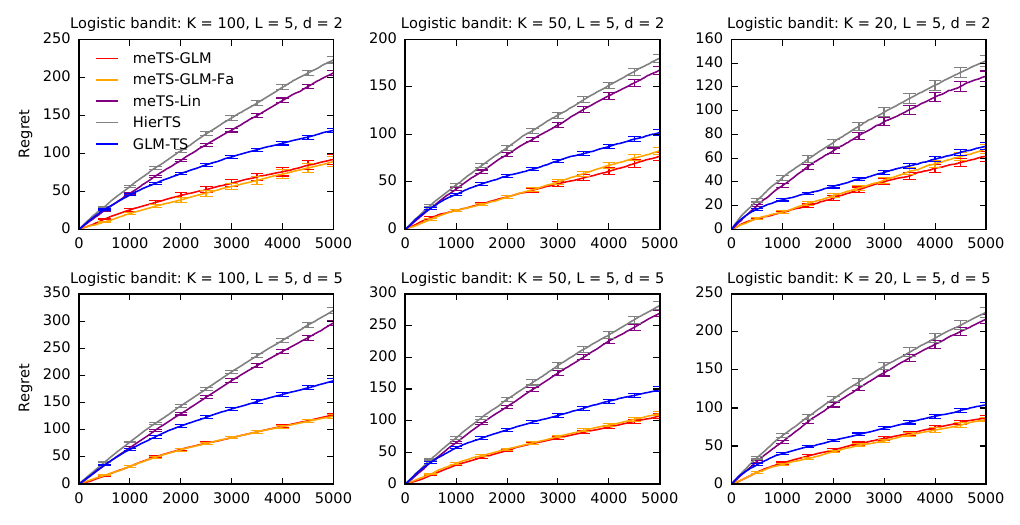}
  \caption{Regret of \algglm on the MovieLens dataset with logistic rewards and varying feature dimension $d \in \{2, 5\}$ and number of actions $K\in \{20, 50, 100\}$.} 
  \label{fig:app_movielens_regret_log}
\end{figure}

We plot the regret of \alg and the baselines up to $n=5000$ rounds in \cref{fig:app_movielens_regret_lin,fig:app_movielens_regret_log}. We observe that \alg outperforms the other baselines. This is despite the fact that we did not fine-tune the mixing weights, which attests to the robustness of our approach to model misspecification. Similarly to the synthetic problems, we observe that the gap in performance between \alg and other baselines is less significant when $K/L$ is small.

\subsection{Robustness to Model Misspecification}\label{app:robustness_model_misspecification}

We conduct additional synthetic experiments where the hyper-parameters do not match the parameters of the bandit environment to assess the robustness of our approach to misspecification. We provide results for this experiment in \cref{fig:misspecified_synthetic_regret}. Here we consider the setting described in \cref{sec:synthetic experiments} except that the true hyper-parameters are misspecified as follows. At each run, we sample uniformly 4 misspecification constants $c_1, c_2, c_3,$ and $c_4$ from $(0, 2)$ and set the hyper-parameters of \alglin as $c_1 \Sigma_{\Psi}$,  $c_2 \mu_{\Psi}$, $c_3 \Sigma_{0, i},$ and $c_4 b_i$ for any $i \in [K]$; where $ \Sigma_{\Psi}$,  $ \mu_{\Psi}$, $ \Sigma_{0, i},$ and $b_i$ for $i \in [K]$ are the true hyper-parameters. Model misspecification is only applied to \alglin and we refer to it as \texttt{meTS-Lin-mis}. We compare it to \alglin and the other baselines, all with the true hyper-parameters. Although the baselines are not misspecified, \texttt{meTS-Lin-mis} still performs better. \texttt{meTS-Lin-mis} also performs similarly to \alglin (with true hyper-parameters).

\begin{figure}
  \centering
  \includegraphics[width=\linewidth]{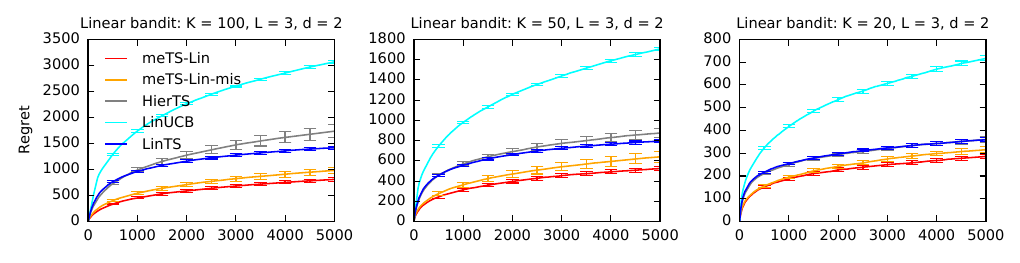}
  \caption{Regret of \emph{misspecified} \alglin on synthetic bandit problems with a varying number of actions $K$. Here, the \emph{misspecified} \alg, \texttt{meTS-Lin-mis}, is compared to baselines with true hyper-parameters.} 
  \label{fig:misspecified_synthetic_regret}
\end{figure}

\subsection{Effect of Action Uncertainty}\label{app:uncertain_action}
As we mentioned in \cref{sec:synthetic experiments}, learning the effect parameters is most beneficial when they are more uncertain than the action parameters. In this section, we support this claim by conducting an experiment where the initial uncertainty of action parameters is greater than the initial uncertainty of the effect parameters.  Precisely, we consider the setting described in \cref{sec:synthetic experiments} except that we set $\Sigma_\Psi=I_{Ld}$ and $\Sigma_{0,i} = 3I_{d}$ for all $i \in [K]$. We report the results in \cref{fig:uncertain_synthetic_regret}. By comparing \cref{fig:uncertain_synthetic_regret} to \cref{fig:synthetic_regret}, we observe that \alglin still outperforms the baselines but the gap in performance shrinks when the action parameters are more uncertain than the effect parameters.

\begin{figure}
  \centering
  \includegraphics[width=\linewidth]{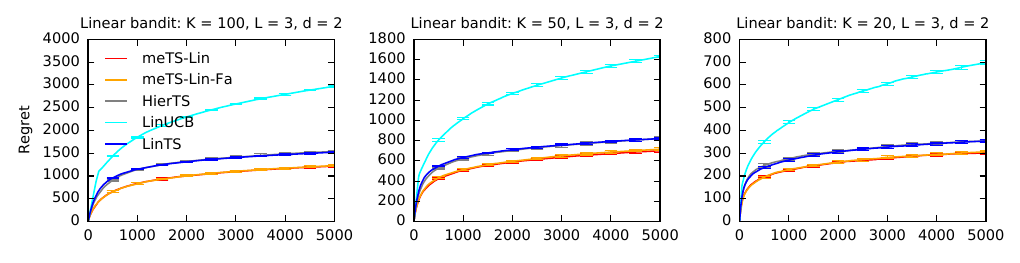}
  \caption{Regret of \alglin on synthetic bandit problems with a varying number of actions $K$, where the action parameters are more uncertain than the effect parameters.} 
  \label{fig:uncertain_synthetic_regret}
\end{figure}

\section{SOCIETAL IMPACT}
The goal of this work is to develop and analyze practical algorithms for contextual bandits with correlated actions. We are not aware of any potential negative societal impacts of our work since we did not propose any new applications of bandit algorithms than existing ones. A typical application of bandit algorithms is recommendation where the preferred items are shown to users. However, by doing this, the recommender system tends to restrain the user to these preferences which may raise concerns regarding the corresponding societal impact.